\documentclass{article}
\PassOptionsToPackage{square,numbers,sort&compress}{natbib}
\usepackage[preprint]{neurips_2026}
\usepackage[utf8]{inputenc} 
\usepackage[T1]{fontenc}    

\usepackage{booktabs}       
\usepackage{amsfonts}       
\usepackage{nicefrac}       
\usepackage{microtype}      
\usepackage{xcolor}         
\usepackage{xspace}
\usepackage{graphicx}
\usepackage{wrapfig}
\usepackage{wrapfig}
\usepackage{booktabs}
\usepackage[hidelinks]{hyperref} 
\usepackage{makecell}   
\usepackage{amsmath, amsthm, amssymb}
\usepackage{algorithm}
\usepackage{algorithmicx}
\usepackage[noend]{algpseudocode}
\usepackage{multirow} 
\usepackage{subcaption} 
\usepackage{caption}
\usepackage{float}  
\usepackage{algorithm}
\usepackage{algpseudocode}
\usepackage[compact]{titlesec}
\usepackage{enumitem}

\theoremstyle{plain}
\newtheorem{theorem}{Theorem}
\newtheorem{proposition}{Proposition}
\newtheorem{lemma}{Lemma}

\newtheorem{corollary}{Corollary}
\newtheorem{problem}{Problem}

\newtheorem{assumption}{Assumption}

\newtheorem{innercustomthm}{Theorem}
\newenvironment{customthm}[1]
  {\renewcommand\theinnercustomthm{#1}\innercustomthm}
  {\endinnercustomthm}

\newtheorem{innercustomcor}{Corollary}
\newenvironment{customcor}[1]
  {\renewcommand\theinnercustomcor{#1}\innercustomcor}
  {\endinnercustomcor}

\newtheorem{innercustomprop}{Proposition}
\newenvironment{customprop}[1]
  {\renewcommand\theinnercustomprop{#1}\innercustomprop}
  {\endinnercustomprop}

\newtheorem*{theorem*}{Theorem}
\newtheorem*{proposition*}{Proposition}
\newtheorem*{lemma*}{Lemma}
\newtheorem*{property*}{Property}
\newtheorem*{definition*}{Definition}
\newtheorem*{corollary*}{Corollary}

\theoremstyle{definition}

\newcommand{\vparagraph}[1]{
  {\noindent \textbf{#1}}
}

\title{Dynamic Shapley Computation}
\author{%
  Xuan Yang$^1$ \quad Hsi-Wen Chen$^2$ \quad Ming-Syan Chen$^2$ \quad Jian Pei$^1$ \\
  $^1$Duke University \quad $^2$National Taiwan University \\
  \texttt{\{xuan.yang, j.pei\}@duke.edu} \quad
  \texttt{\{hwchen, mschen\}@ntu.edu.tw}
}
\begin{document}

\maketitle
\begin{abstract}
Shapley-based data valuation provides a principled way to quantify the contribution of training data, but its high computational cost makes it impractical in dynamic settings where tasks and training players evolve. Existing methods treat Shapley computation as a one-shot process and collapse contributions into aggregated scores, preventing reuse and requiring recomputation under any change. We introduce a new perspective that represents Shapley values as a \textbf{player-by-task matrix} and formulates dynamic valuation as a \textbf{structured matrix maintenance problem}. We exploit the fact that each task depends on a small subset of training players and that similar tasks yield similar valuations, leading to \textbf{utility locality} and \textbf{coalition locality}. Based on these insights, we propose \textbf{D-Shap}, a dynamic valuation framework that enables efficient updates by modifying only a small portion of the matrix: new task valuations are inferred via structure-aware interpolation, while updates induced by new players are confined to affected local matrix blocks. To eliminate the need for pre-specified evaluation tasks, we introduce \textbf{self-valuation}, which constructs the initial matrix directly from training data, supported by scalable subset reuse and coverage-aware anchor selection. Experiments across diverse models show that D-Shap performs task updates in milliseconds and reduces the cost of player updates by up to three orders of magnitude, while achieving valuation quality competitive with full recomputation.
\end{abstract}
\section{Introduction}
\label{sec:intro}

The \textbf{Shapley value}~\cite{shapley1953value} has become a standard tool for data valuation in machine learning~\cite{ghorbani2019data}, supporting applications such as data pricing~\cite{liu2021dealer}, dataset curation~\cite{bhardwaj2024state}, and model debugging~\cite{adebayo2020debugging}. In these applications, Shapley value evaluates the contribution of each \textbf{contributor} (player) to a given \textbf{task}. Here, a \textbf{task} refers to an evaluation objective defined by a utility function, such as model performance on a query, a subset of data, or an application-specific criterion~\cite{jia2019towards}\cite{liu2022gtg}. In practice, tasks can be instantiated by individual queries, batches of queries, or downstream evaluation workloads~\cite{Lundberg2020TreeSHAP,jiang2023opendataval}.

Modern applications are inherently \emph{dynamic}: new tasks continuously arise, and new data contributors expand the training set~\cite{de2021continual,parisi2019continual}. As a result, \textbf{dynamic Shapley value computation}—updating valuations efficiently under evolving tasks and training data—becomes a fundamental requirement.

At a high level, Shapley value quantifies the contribution of each training player by averaging its marginal contribution over all subsets~\cite{shapley1953value}. In dynamic settings, this leads to two core challenges: (i) \textbf{task-incremental updates}, where valuations must be computed for newly arriving tasks, and (ii) \textbf{player-incremental updates}, where new training players alter the underlying cooperative game and affect existing valuations. In practice, the two types of updates often occur \emph{simultaneously}, requiring the valuation method to handle evolving tasks and players in a unified manner. Existing methods are fundamentally \emph{static}: they assume fixed data and fixed tasks, and therefore must recompute from scratch for every update~\cite{ghorbani2019data}\cite{sun2024shapley}. The few works on dynamic Shapley~\cite{zhang2023dynamic,xia2025computing} handle limited scenarios and still rely on global recomputation, making them unable to scale to realistic workloads.

The root cause is the prohibitive cost of Shapley computation. Exact computation is \#P-hard~\cite{deng1994complexity}, requiring exponentially many coalition evaluations. Even approximate methods~\cite{ghorbani2019data,Castro2009Polynomial,koh2017understanding,basu2020influence,wang2024data} remain expensive, as each evaluation typically involves training or updating a machine learning model. In dynamic settings, repeatedly retraining models for every update is infeasible. More fundamentally, existing approaches reduce Shapley values to a single aggregated score per training player~\cite{ghorbani2019data,wang2023data}, discarding how contributions vary across different tasks. This loss of information prevents reuse and forces global recomputation.

This paper introduces a fundamentally different perspective: \textbf{Shapley values should be represented as a collection of per-task contributions that can be maintained and reused}. Instead of treating valuation as a one-shot computation, we explicitly model how each training player contributes to each task. Our key insight is that, for a given task, the model typically depends only on a small subset of relevant training instances~\cite{yang2026local,yeh2018representer,xu2018representation,jia2019efficient}. This \textbf{local dependency} implies that both task-incremental and player-incremental updates can be performed by focusing only on the affected subsets, rather than recomputing globally.

Building on this insight, we propose \textbf{D-Shap}, a new framework for \textbf{dynamic Shapley value computation}. D-Shap represents Shapley values as a \textbf{player-by-task matrix}, where each entry records the contribution of a training instance to a task. This representation preserves fine-grained information that is lost in prior aggregation-based methods and enables two key operations: (i) \textbf{task-incremental updates}, which estimate valuations for new tasks by leveraging similar existing tasks, and (ii) \textbf{player-incremental updates}, which update only the affected parts of the matrix when new training players arrive. By maintaining and reusing this matrix, D-Shap avoids repeated global recomputation.\footnote{D-Shap also supports deletion and replacement: task-side updates remove or replace the corresponding Shapley-matrix column, while player-side updates identify the coalitions affected by the player change and update them selectively.}
\begin{figure}[t]
  \centering
  \includegraphics[width=.96\textwidth]{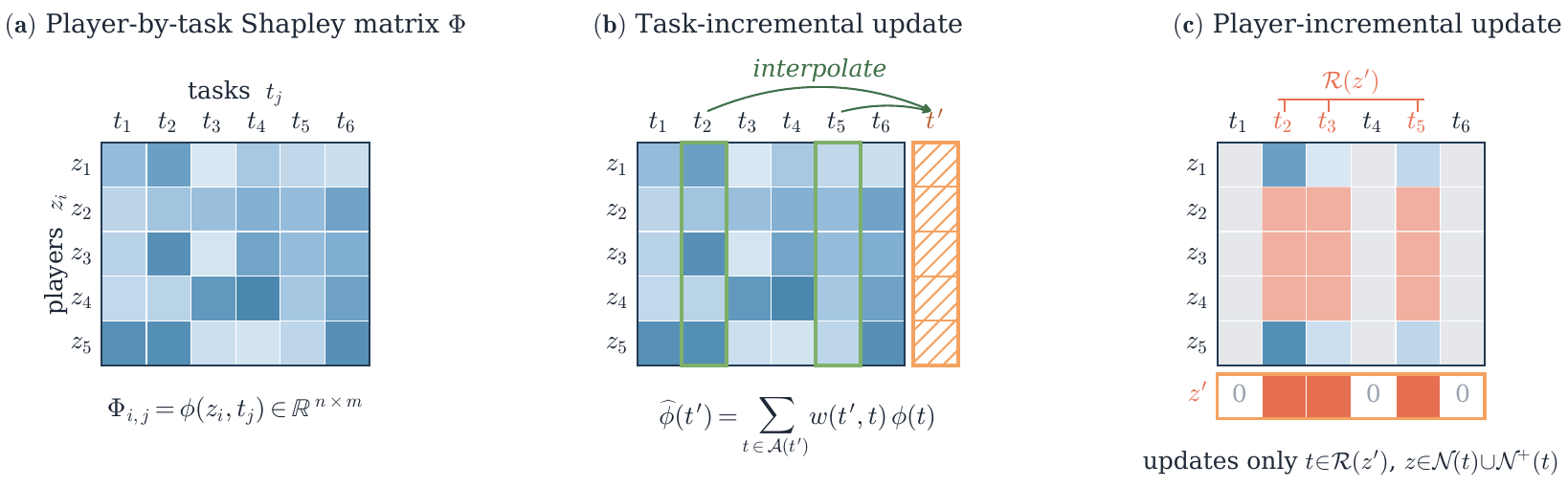}
    \caption{\textbf{Overview of D-Shap.}
    \textbf{(a)} D-Shap represents Shapley values as a player-by-task matrix $\boldsymbol{\Phi}$, where $\boldsymbol{\Phi}_{i,j}=\phi(z_i,t_j)$ denotes player $z_i$'s contribution to task $t_j$.
    \textbf{(b)} \emph{Task-incremental update} (Sec.~\ref{sec:test}): a new task $t'$ is valued by interpolating nearby columns (green) under the model-induced distance $d_\Gamma$.
    \textbf{(c)} \emph{Player-incremental update} (Sec.~\ref{sec:player}): a new player $z'$ affects only columns whose local games involve it (coral); other entries (gray) stay frozen.}
\label{fig:overview}
\end{figure}

This perspective transforms Shapley computation from an exponential, one-shot procedure into a \textbf{reusable and incremental process}. It opens a new direction for scalable data valuation in dynamic machine learning systems, making Shapley-based methods practical in settings that were previously computationally prohibitive.
\section{Dynamic Shapley Computation}
\label{sec:problem}

\vparagraph{Shapley value for data valuation in Machine Learning.}
Let $\mathcal{D}=\{z_i\}_{i=1}^{n}$ be a set of \textbf{training contributors} (players), and let $\mathcal{T}=\{t_j\}_{j=1}^{m}$ be a set of \textbf{tasks}. Each task $t\in\mathcal{T}$ specifies an evaluation objective through a utility function.

Given a task $t$, the utility of a subset $\mathcal{S}\subseteq\mathcal{D}$ is $v_t(\mathcal{S}) = g(\theta(\mathcal{S}), t)$, where $\theta(\mathcal{S})$ is the model trained on $\mathcal{S}$, and $g(\cdot,t)$ evaluates the model on task $t$ (e.g., loss, accuracy, or confidence)~\cite{ghorbani2019data,jia2019efficient}.

The \textbf{Shapley value}~\cite{shapley1953value,osborne1994course,ghorbani2019data} of training player $z\in\mathcal{D}$ with respect to task $t$ is
\begin{equation}
\label{eq:value}
\phi(z,t)
=
\frac{1}{n}
\sum_{\mathcal{S}\subseteq \mathcal{D}\setminus\{z\}}
\frac{
v_t(\mathcal{S}\cup\{z\})-v_t(\mathcal{S})
}{
\binom{n-1}{|\mathcal{S}|}
}.
\end{equation}
It measures the average marginal contribution of $z$ to task $t$. However, computing $\phi(z,t)$ requires evaluating exponentially many subsets, making it \#P-hard~\cite{deng1994complexity}.

\vparagraph{The player-by-task Shapley matrix.}
Existing methods can compute $\phi(z,t)$ for each player--task pair. However, they typically \emph{aggregate} these values across tasks into a single score per player, or treat each task independently in a one-shot manner. As a result, the per-task contributions are not explicitly maintained or reused. This lack of a persistent representation fundamentally limits reuse: when tasks or players change, Shapley values must be recomputed from scratch.

Our key idea is to retain this information by representing Shapley values as a \textbf{player-by-task matrix},
\begin{equation}
\label{eqn:shapley_matrix}
\boldsymbol{\Phi} \in \mathbb{R}^{n\times m}, \quad
\boldsymbol{\Phi}_{i,j} = \phi(z_i,t_j),
\end{equation}
where each entry records the contribution of player $z_i$ to task $t_j$.

This matrix has a clear interpretation: each \textbf{row} $\boldsymbol{\Phi}_{i,:}$ captures how a player performs across tasks, while each \textbf{column} $\boldsymbol{\Phi}_{:,j}$ represents the Shapley values for a specific task. By explicitly modeling \emph{who contributes to which task}, this representation enables reuse across both players and tasks, which is not possible under aggregated formulations.

\vparagraph{Dynamic Shapley valuation.}
We study \textbf{dynamic updates} where tasks and players evolve over time.

\begin{problem}[\textbf{Task-Incremental Valuation}]
\label{prob:test_incremental}
Given a Shapley matrix $\boldsymbol{\Phi}$ and a set of new tasks
$\Delta \mathcal{T}=\{t_{m+1},\ldots,t_{m+\Delta m}\}$, estimate
\begin{equation}
[\boldsymbol{\Phi}_{:,1:m}
\mid
\boldsymbol{\Phi}^{\text{new}}_{:,\,m+1:m+\Delta m}]
\in \mathbb{R}^{n\times (m+\Delta m)},
\end{equation}
where $\boldsymbol{\Phi}^{\text{new}}$ contains the columns for the new tasks.
\end{problem}

Here, the player set is fixed, and the goal is to compute Shapley values for new tasks without recomputing existing ones.

\begin{problem}[\textbf{Player-Incremental Valuation}]
\label{prob:player_incremental}
Given a Shapley matrix $\boldsymbol{\Phi}$ and a set of new players
$\Delta \mathcal{D}=\{z_{n+1},\ldots,z_{n+\Delta n}\}$, estimate
\begin{equation}
\begin{bmatrix}
\boldsymbol{\Phi}^{\text{upd}}_{1:n,:} \\
\boldsymbol{\Phi}^{\text{new}}_{\,n+1:n+\Delta n,:}
\end{bmatrix}
\in \mathbb{R}^{(n+\Delta n)\times m}.
\end{equation}
\end{problem}

In this case, the cooperative game changes, requiring both new rows (for new players) and updates to existing rows. 

Our goal is to maintain the Shapley matrix incrementally with fewer model trainings than full recomputation by treating it as a \textbf{maintained object} and updating only the necessary parts rather than recomputing values from scratch. In practice, \textbf{task-incremental} and \textbf{player-incremental} updates may occur simultaneously, which requires a unified solution.
\section{D-Shap}
\label{sec:method}

Given the \textbf{player-by-task Shapley matrix} in Section~\ref{sec:problem}, our goal is to \textbf{maintain this matrix under dynamic updates}. Dynamic Shapley computation thus reduces to updating $\boldsymbol{\Phi}$ when new tasks arrive (adding columns) and when new players are introduced (adding rows and modifying affected entries). Rather than recomputing all values, \textbf{D-Shap} updates only a small portion of $\boldsymbol{\Phi}$ by exploiting the fact that each task depends on a limited subset of players and that similar tasks exhibit similar contribution patterns.

\subsection{D-Shap Framework}
\label{sec:framework}

Let $\boldsymbol{\Phi}\in\mathbb{R}^{n\times m}$ denote the maintained Shapley matrix. We consider two update operators: (i) a \textbf{task update}, which takes a new task $t'$, computes its Shapley column $\boldsymbol{\phi}(t')$, and appends it to $\boldsymbol{\Phi}$; and (ii) a \textbf{player update}, which takes a new player $z'$, adds a new row for $z'$, and updates only the entries whose values change under the expanded player set.

The central challenge is to identify the \textbf{affected entries} for each update. For a task $t$, let $\mathcal{N}(t)\subseteq\mathcal{D}$ denote the subset of training players that determine its utility. This induces a sparse dependency: $\phi(z,t)$ is nonzero only if $z\in\mathcal{N}(t)$~\cite{yang2026local}. Consequently, a task update requires estimating $\boldsymbol{\phi}(t')$ from related tasks, while a player update modifies only the columns whose dependency sets change.

This view reduces dynamic Shapley computation to \textbf{localized matrix updates}. In the following subsections, we formalize this dependency and derive efficient procedures for task-incremental updates, player-incremental updates, and matrix construction.

\subsection{Model-Induced Locality}
\label{sec:locality}

A key observation enabling efficient dynamic Shapley computation is that, in many machine learning models, the utility of a task is dominated by a \textbf{small subset of training players}~\cite{yeh2018representer,mirzasoleiman2020coresets,xu2018representation}. This phenomenon, referred to as \textbf{model-induced locality}~\cite{yang2026local}, implies that Shapley values are inherently sparse and structured, and therefore do not require global recomputation under updates. In contrast, recomputing Shapley values directly would revisit all players and tasks, with each entry defined by averaging marginal contributions over exponentially many coalitions~\cite{shapley1953value,ghorbani2019data}.

Following Yang et al.~\cite{yang2026local}, we formalize this property through a \textbf{local support set}. For each task $t$, let $\mathcal{N}(t)\subseteq\mathcal{D}$ denote the subset of players that influence its utility and thus may receive nonzero Shapley value.\footnote{If the utility is only approximately determined by $\mathcal{N}(t)$ with error at most $\epsilon$, the induced Shapley error is bounded by $2\epsilon$; see Lemma~\ref{lem:locality} in Appendix~\ref{proof:locality}.} This notion captures how the model uses training data for a given task and naturally varies across model classes. For example, in KNN, $\mathcal{N}(t)$ consists of nearest neighbors~\cite{mairal2008supervised,jia2019efficient}; in decision trees, training players in the same leaf~\cite{Lundberg2020TreeSHAP}; in kernel methods, players with high similarity~\cite{cortes1995support}; and in graph neural networks, nodes within a local neighborhood~\cite{kipf2017gcn,hamilton2017inductive}.

The existence of small support sets has two important consequences for maintaining the Shapley matrix. First, tasks with similar supports tend to have similar contribution patterns, enabling efficient estimation of new columns. Second, a new player can affect only the tasks whose support sets change, restricting updates to a small subset of entries. These effects give rise to \textbf{task-side locality} and \textbf{player-side locality}, which form the foundation of our incremental update methods.

Together, model-induced locality reduces dynamic Shapley computation from a global problem to a \textbf{localized update problem}, where only a small portion of the matrix needs to be recomputed.
\subsection{Task-Incremental Valuation via Utility Locality}
\label{sec:test}

We first address Problem~\ref{prob:test_incremental}. Given a Shapley matrix $\boldsymbol{\Phi}$, a new task $t' = t_{m+1}$ corresponds to adding a new column $\boldsymbol{\phi}(t')$. The goal is to estimate this column efficiently without recomputing Shapley values from scratch. The key challenge is to identify existing tasks that are informative for estimating $\boldsymbol{\phi}(t')$. This is enabled by \textbf{utility locality}: tasks with similar local dependencies induce similar utility functions and therefore have similar Shapley columns.

A natural baseline is to use feature-space proximity to find similar tasks~\cite{jethani2021fastshap,sun2026fast}. However, feature similarity does not capture how the model uses training data. For example, in a graph neural network, two nodes with similar attributes but different neighborhoods induce different local computation structures, leading to very different Shapley values~\cite{hamilton2017inductive,yuan2021explainability,duval2021graphsvx}. As a result, feature-based similarity can produce misleading estimates.

To address this, we compare tasks based on their \textbf{local computation structure}. Let $\Gamma(t)$ denote the local computation structure of task $t$, which captures both the supporting players $\mathcal{N}(t)$ and how their contributions are combined by the model (e.g., through weights, aggregation, or connectivity). This structure reflects how the model computes the utility of $t$ from its relevant players. Examples include decision paths in trees, message-passing neighborhoods in GNNs, weighted neighbor profiles in KNN, and kernel-relevance patterns in kernel models. 

We define a model-induced distance
\begin{equation}
d_\Gamma(t,t') = d\big(\Gamma(t), \Gamma(t')\big),
\end{equation}
where $d_\Gamma$ is instantiated according to the model (e.g., path overlap, graph distance, or support overlap); see Appendix~\ref{apx:model_locality} for concrete instantiations across different model families. A small value of $d_\Gamma(t,t')$ indicates that $t$ and $t'$ share similar computation structures and therefore exhibit \textbf{utility locality}.

To connect this notion of similarity—captured by the model-induced distance $d_\Gamma$ between local computation structures—to Shapley values, we impose a stability condition on the evaluation function.

\begin{assumption}[Model-induced Lipschitz continuity]
\label{asm:lipschitz}
For every coalition $S \subseteq \mathcal{D}$, the trained model $\theta(S)$ lies in a hypothesis class $\Theta$. Moreover, for any fixed model $\theta \in \Theta$, the evaluation function $g(\theta, t)$ varies smoothly across tasks with respect to the model-induced distance $d_\Gamma$, i.e.,
\[
|g(\theta, t) - g(\theta, t')|
\le
L_\Gamma d_\Gamma(t, t')
\quad \text{for all } t, t'.
\]

\end{assumption}

Assumption~\ref{asm:lipschitz} ensures that tasks with similar computation structures induce similar utility values under the same model. This condition holds for regularized ERM models and more broadly for models with bounded parameters and Lipschitz losses~\cite{bousquet2002stability,shalev2010learnability}; see Appendix~\ref{app:erm_utility_locality} for a derivation showing how these conditions imply the Lipschitz bound.

Under this condition, Shapley columns inherit the same smoothness~\cite{jia2019towards,ghorbani2020distributional}.

\begin{proposition}[Utility locality]
\label{prop:test}
Under Assumption~\ref{asm:lipschitz}, for any two test points $t$ and $t'$,
\begingroup\small
\[
\|\boldsymbol{\phi}(t)-\boldsymbol{\phi}(t')\|_\infty
\le
2L_\Gamma d_\Gamma(t,t').
\]
\endgroup

\end{proposition}

Proposition~\ref{prop:test} shows that tasks that are close under $d_\Gamma$ have similar Shapley columns. This directly enables \textbf{column interpolation}: we estimate $\boldsymbol{\phi}(t')$ by combining nearby columns,
\begin{equation}
\label{eqn:task_update}
\widehat{\boldsymbol{\phi}}(t')
=
\sum_{t\in \mathcal{A}(t')}
w(t',t)\boldsymbol{\phi}(t),
\end{equation}
where $\mathcal{A}(t')$ contains tasks nearest to $t'$ under $d_\Gamma$, and the weights are convex.

We formalize the accuracy of this estimator as follows.

\begin{theorem}[Interpolation error]
\label{thm:test}
Under Assumption~\ref{asm:lipschitz}, if every reference point $t\in \mathcal{A}(t')$ satisfies $d_\Gamma(t',t) \leq \varepsilon$, then
\begingroup\small
\[
\big\|
\widehat{\boldsymbol{\phi}}(t')
-
\boldsymbol{\phi}(t')
\big\|_\infty
\leq
2L_\Gamma  \varepsilon.
\]
\endgroup

\end{theorem}

Theorem~\ref{thm:test} shows that the estimation error is controlled by the distance between $t'$ and its neighbors under $d_\Gamma$. This yields an efficient update rule: \textbf{new columns are inferred from nearby columns rather than recomputed}, reducing exponential computation to a local interpolation step.

The same procedure applies independently to each task in $\Delta\mathcal{T}$, producing $\boldsymbol{\Phi}^{\mathrm{new}}_{:,\,m+1:m+\Delta m}$.
\subsection{Player-Incremental Valuation via Coalition Locality}
\label{sec:player}

We next address Problem~\ref{prob:player_incremental}. Given a Shapley matrix $\boldsymbol{\Phi}$, the arrival of a new player $z' = z_{n+1}$ expands the matrix by adding a new row $\boldsymbol{\Phi}^{\mathrm{new}}_{n+1,:}$ and may also change existing entries $\boldsymbol{\Phi}^{\mathrm{upd}}_{1:n,:}$. A naive approach would recompute all entries, which is prohibitively expensive.

The key observation is that these updates are inherently \textbf{localized}. For each task $t$, its utility depends only on a subset of players $\mathcal{N}(t)$, which we refer to as its \textbf{local dependency structure}. This induces a \textbf{local cooperative game} restricted to players in $\mathcal{N}(t)$, which fully determines the Shapley values for task $t$. When a new player $z'$ is added, only the tasks whose local dependency structures change need to be updated.

Formally, let $\mathcal{N}(t)$ and $\mathcal{N}^{+}(t)$ denote the support sets of task $t$ before and after inserting $z'$. The set of \textbf{affected tasks} is
\begin{equation}
\label{eq:new_player_affected_set}
\mathcal{R}(z') = \{t \in \mathcal{T} : \mathcal{N}^{+}(t) \neq \mathcal{N}(t)\},
\end{equation}
i.e., the tasks whose local dependency structures, and hence local games, change.

Let $\phi^{+}(z,t)$ and $\phi(z,t)$ denote the Shapley values under the updated and original games, respectively. The following result formalizes the coalition locality.

\begin{proposition}[Coalition locality]
\label{prop:player}
If the model evaluation at $t$ depends only on its local support, then for any $t \notin \mathcal R(z')$, the valuation at $t$ remains unchanged, i.e., $\phi^{+}(z_j,t)=\phi(z_j,t)$ for all $z_j\in\mathcal D$.

\end{proposition}

Proposition~\ref{prop:player} shows that only tasks in $\mathcal{R}(z')$ require updates. Moreover, within each affected task, only players in $\mathcal{N}(t)\cup\mathcal{N}^{+}(t)$ can change. All other entries remain unchanged and are directly reused. Thus, updates are restricted to a \textbf{small local block} of the matrix.

We summarize the update rule as follows.

\begin{theorem}[Local player update]
\label{thm:player}
Under coalition locality, for any $t \in \mathcal R(z')$, the updated Shapley values are determined by the new local game restricted to $\mathcal{N}^{+}(t)$:
\begingroup\small
\[
\phi^{+}(z,t)
=
\begin{cases}
\displaystyle
\frac{1}{|\mathcal{N}^{+}(t)|}
\sum_{S \subseteq \mathcal{N}^{+}(t)\setminus\{z\}}
\frac{v_t(S \cup \{z\}) - v_t(S)}
{\binom{|\mathcal{N}^{+}(t)|-1}{|S|}},
& z \in \mathcal{N}^{+}(t), \\[12pt]
0,
& z \in \mathcal{N}(t)\setminus \mathcal{N}^{+}(t), \\[6pt]
\phi(z,t),
& z \notin \mathcal{N}(t)\cup\mathcal{N}^{+}(t).
\end{cases}
\]
\endgroup
\end{theorem}

The theorem partitions players into three groups: (i) players in $\mathcal{N}^{+}(t)$, which are evaluated under the updated local game; (ii) players removed from the support, whose contributions become zero; and (iii) unaffected players, whose values remain unchanged. This decomposition makes the update entirely local.

A particularly efficient case arises when the support only expands.

\begin{corollary}[Monotone support expansion]
\label{cor:monotone}
For any $t \in \mathcal{R}(z')$, if $\mathcal{N}^{+}(t) = \mathcal{N}(t) \cup \{z'\}$, then the local update can reuse the original Shapley matrix. Specifically, $z'$ is valued on $\mathcal{N}^{+}(t)$ using coalitions from $\mathcal{N}(t)$, since $\mathcal{N}^{+}(t)\setminus\{z'\}=\mathcal{N}(t)$. Each existing player $z_j \in \mathcal{N}(t)$ is updated by
\begingroup\small
\[
\phi^{+}(z_j, t)
=
\phi(z_j, t) + \Delta_{j,t}(z').
\]
\endgroup
Here, $\Delta_{j,t}(z')$ captures the change in the marginal contribution of $z_j$ caused by inserting $z'$ on $t$.
\end{corollary}

Corollary~\ref{cor:monotone} shows that when $\mathcal{N}^{+}(t) = \mathcal{N}(t) \cup \{z'\}$, the update can \textbf{reuse existing Shapley values}. The new row is computed from the marginal effect of $z'$, while existing players receive an additive correction that accounts for how $z'$ changes their contributions.\footnote{We assume that $\mathcal{N}(t)$ fully determines the utility of $t$; the approximate case is given in Theorem~\ref{thm:player_approx} in Appendix~\ref{proof:player_approx}.}

In this case, the number of new evaluations is bounded by
\begin{equation}
\sum_{t\in\mathcal{R}(z')}2^{|\mathcal{N}(t)|}
\le
|\mathcal{R}(z')|2^{K_{\max}},
\end{equation}
where $K_{\max}=\max_{t\in\mathcal{R}(z')}|\mathcal{N}(t)|$; see Corollary~\ref{cor:cost} in Appendix~\ref{proof:cost}.

Overall, inserting new players expands $\boldsymbol{\Phi}$ by new rows, while existing entries are updated only for affected tasks and local dependency structures. Since these local games are typically much smaller than the full player set~\cite{jia2019efficient,yang2026local}, the cost depends on local structure size rather than the entire matrix, turning global recomputation into efficient localized updates.

\vparagraph{General Dynamic Updates.}
D-Shap also supports deletion under the same matrix-maintenance principle. For task deletion, removing a task simply deletes its corresponding column from $\boldsymbol{\Phi}$. For player deletion, removing a player $z$ deletes its row from $\boldsymbol{\Phi}$ and updates only the tasks whose local support sets are affected. Let $\mathcal{N}^{-}(t)$ denote the support set after deleting $z$. If $\mathcal{N}^{-}(t)=\mathcal{N}(t)$, coalition locality implies that the remaining entries are unchanged. Otherwise, D-Shap recomputes only the local game induced by $\mathcal{N}^{-}(t)$ while reusing all other entries.

More broadly, simultaneous task and player updates are handled by composing task-side interpolation with player-side localized row/block updates. The new task columns are estimated using the interpolation guarantee in Theorem~\ref{thm:test}, while entries affected by newly added players are updated using the local-game principle in Theorem~\ref{thm:player}. Their interaction is confined to new task columns whose support sets include new players. Therefore, task/player addition, deletion, replacement, and joint updates are unified as localized Shapley-matrix maintenance; see Appendix~\ref{apx:dynamic}.
\subsection{Constructing the Shapley Matrix via Self-Valuation}
\label{sec:self_valuation}

The effectiveness of D-Shap relies on an initial \textbf{player-by-task Shapley matrix} that can support future updates. However, in dynamic settings, such a matrix is often unavailable: an external set of tasks may not exist at initialization time, or may be poorly aligned with future tasks, leading to inaccurate interpolation and degraded performance.

Therefore, we introduce \textbf{self-valuation}, a new mechanism that constructs the Shapley matrix \emph{directly from the training data itself}, without requiring any external tasks. The key idea is to treat each player as a \textbf{proxy task}, thereby turning the training set into a self-contained set of evaluation tasks.

Specifically, for each player $z_i \in \mathcal{D}$, we define a leave-one-out game with player set $\mathcal{D}_{-i} = \mathcal{D} \setminus \{z_i\}$ and treat $z_i$ as a task. This yields a \textbf{self-valuation Shapley matrix} $\boldsymbol{\Phi} \in \mathbb{R}^{n \times n}$ with entries $\phi(z_j, z_i)$ computed under $\mathcal{D}_{-i}$. Diagonal entries are undefined since $z_i$ is excluded. This construction avoids the degenerate case where a player explains itself, while producing a matrix that reflects intrinsic relationships within the data.

\paragraph{Shared Subset Scheduling.}
A naive construction would solve $n$ independent leave-one-out Shapley games, each requiring $\Theta(2^{n-1})$ model trainings, resulting in $\Theta(n \cdot 2^{n-1})$ total cost. This is prohibitively expensive.

To eliminate redundancy, we introduce \textbf{shared subset scheduling}, which reorganizes computation around coalition reuse. Instead of evaluating each coalition separately for each task, we assign every coalition $S \subseteq \mathcal{D}$ a unique \textbf{pivot} $z_{\pi(S)} \in \mathcal{D} \setminus S$. The model $\theta(S)$ is trained once at its pivot and its utility is reused across all tasks $z_i \notin S$, contributing to $\phi(\cdot, z_i)$ with appropriate weights. 

This design evaluates each coalition exactly once while preserving the exact Shapley estimator. Thus, self-valuation requires no more training runs than a single global Shapley computation, yet produces $n$ valuation columns. This transforms an otherwise prohibitive cost into a practical one.

\paragraph{Coverage-Aware Anchor Selection.}
Even with shared subset scheduling, constructing the full matrix can be expensive for large datasets. We therefore build a compact Shapley matrix using a small set of representative players, called \textbf{anchors}. Specifically, we select $\{a_1,\ldots,a_k\}\subseteq\mathcal{D}$ with $k\ll n$ and construct $\boldsymbol{\Phi}\in\mathbb{R}^{n\times k}$. Anchor quality is measured by the \textbf{covering radius}
\begin{equation}
\label{eq:covering_radius}
r_{\max} = \max_{z_i \in \mathcal{D}} \min_{a \in \{a_1,\ldots,a_k\}} d_\Gamma(z_i, a).
\end{equation}
A smaller radius improves coverage and makes interpolation more reliable. We minimize this radius using farthest-point sampling, which gives a standard $2$-approximation~\cite{gonzalez1985clustering}. During online updates, if a new task $t'$ is poorly covered, i.e., $\min_a d_\Gamma(t', a) > \tau$, we explicitly compute its Shapley column and add it as a new anchor, enabling \textbf{adaptive matrix expansion} to maintain accuracy over time.

Overall, self-valuation provides a \textbf{self-contained and scalable initialization} of the Shapley matrix. Combined with shared subset scheduling and adaptive anchor selection, it enables D-Shap to operate without external tasks while maintaining both efficiency and accuracy.
\section{Experiments}
\label{sec:experiments}

We evaluate \textbf{D-Shap} across five representative model families with distinct locality mechanisms and compare it against both static and dynamic baselines. Our goal is to examine whether D-Shap maintains high-quality Shapley values while avoiding expensive recomputation in dynamic settings. The evaluation is organized around four questions: \textbf{RQ1 (Task-incremental)} asks whether D-Shap can accurately and efficiently estimate Shapley values for new tasks; \textbf{RQ2 (Player-incremental)} asks whether D-Shap can efficiently update valuations when new players arrive; \textbf{RQ3 (Locality)} studies whether model-induced locality is necessary for reliable valuation updates; and \textbf{RQ4 (Scalability)} evaluates whether constructing the Shapley matrix is efficient and scalable in practice.

\subsection{Experimental Setup}
\label{sec:exp_setup}

\paragraph{Datasets and models.}
We evaluate D-Shap across five representative model–dataset pairs that exhibit diverse forms of model-induced locality, where each model defines locality through a different notion of supporting players and computation structure. We consider: (i) Weighted $K$-Nearest Neighbors (WKNN)~\cite{dudani1976distance} on MNIST~\cite{lecun1998gradient} (neighbor-based locality); (ii) Decision Trees (DT)~\cite{breiman1984classification} on Iris~\cite{fisher1936iris} (partition-based locality); (iii) RBF Kernel SVM~\cite{cortes1995support} on Breast Cancer~\cite{street1993nuclear} (kernel-relevance locality); (iv) Convolutional Neural Networks (CNN)~\cite{lecun1998gradient} on MNIST (representation-space locality); and (v) Graph Convolutional Networks (GCN)~\cite{kipf2017gcn} on Cora~\cite{mccallum2000automating} (graph neighborhood locality). These models span both classical and deep learning settings, highlighting the generality of our framework. Dataset statistics and hyperparameters are provided in Appendix~\ref{apx:exp:setup}.

\vparagraph{Baselines.}
We compare against both static and dynamic baselines. Static methods include \emph{Global-MC Recompute} (high-budget Monte Carlo reference), \emph{TMC-Shapley}~\cite{ghorbani2019data}, and \emph{Comple-S}~\cite{sun2024shapley}. For task-incremental settings, we include learning-based baselines \emph{Fast-DataShapley}~\cite{sun2026fast} and \emph{Amortized-S}~\cite{covert2024stochastic}, which predict Shapley scores without explicit update rules. For player-incremental settings, we compare with \emph{B-Delta}~\cite{xia2025computing}. We exclude \emph{DeltaShap}~\cite{zhang2023dynamic} as it does not finish within our compute budget. These baselines represent state-of-the-art approaches for static estimation, amortized prediction, and incremental updates. 

\vparagraph{Evaluation protocol.}
We construct the initial Shapley matrix using self-valuation, making the setup fully self-contained without external tasks. We then reserve a held-out pool to simulate dynamic updates. In the task-incremental setting, held-out points arrive as new tasks, while the player set remains fixed. In the player-incremental setting, they arrive as new players, while evaluation tasks remain fixed. The held-out pool is processed as a stream, with updates performed online after each arrival. We reserve 30\% of the data for Iris and Breast Cancer, and 1{,}000 points for MNIST and Cora.

\vparagraph{Metrics.}
We evaluate both accuracy and efficiency. Accuracy is measured by Spearman rank correlation $\rho$ (filtering $<10^{-3}$ entries dominated by MC noise) and Pearson correlation $r$ against a high-budget Global-MC reference. Efficiency is measured by wall-clock time $T$ per update~\cite{zhang2023dynamic}. This setup directly evaluates whether D-Shap achieves accurate valuation while avoiding recomputation.

\vparagraph{Implementation details.}
All Monte Carlo methods use a shared stopping criterion and are capped at 5{,}000 samples~\cite{ghorbani2019data}. Convergence is checked every 100 samples using relative change. Each experiment is repeated with five random seeds. Experiments are run on a single machine with 64 Intel Xeon E5-2640 v4 CPUs and 64\,GB RAM; CNN and GNN models additionally use four NVIDIA RTX A5000 GPUs.

%

\begin{table}[t]
    \centering
    \small
    \caption{Task-Incremental Valuation}
    \label{tabtest}
    \resizebox{\textwidth}{!}{%
    \begin{tabular}{lccccccccccccccc}
        \toprule
        \multirow{2}{*}{Method}
        & \multicolumn{3}{c}{WKNN/MNIST}
        & \multicolumn{3}{c}{DT/Iris}
        & \multicolumn{3}{c}{SVM/BC}
        & \multicolumn{3}{c}{CNN/MNIST}
        & \multicolumn{3}{c}{GNN/Cora} \\
        \cmidrule(lr){2-4}\cmidrule(lr){5-7}\cmidrule(lr){8-10}\cmidrule(lr){11-13}\cmidrule(lr){14-16}
        & $\rho$ & $r$ & $T$
        & $\rho$ & $r$ & $T$
        & $\rho$ & $r$ & $T$
        & $\rho$ & $r$ & $T$
        & $\rho$ & $r$ & $T$ \\
        \midrule
        Global-MC$^\dagger$
        & 1.000 & 1.000 & 2.0e4
        & 1.000 & 1.000 & 2.1e3
        & 1.000 & 1.000 & 1.2e4
        & 1.000 & 1.000 & 2.1e5
        & 1.000 & 1.000 & 1.2e5 \\
        TMC-Shapley
        & 0.881 & 0.858 & 6.9e3
        & 0.805 & 0.906 & 1.6e3
        & 0.776 & 0.850 & 5.6e3
        & 0.678 & 0.921 & 9.6e4
        & 0.788 & 0.947 & 1.3e5 \\
        Comple-S
        & 0.502 & 0.750 & 2.6e3
        & 0.821 & 0.928 & 1.7e2
        & 0.648 & 0.782 & 3.7e3
        & 0.708 & 0.937 & 5.4e4
        & 0.804 & 0.952 & 8.6e4 \\
        \midrule
        Amortized-S
        & 0.114 & 0.106 & 2.2e-2
        & 0.594 & 0.647 & 3.1e-2
        & 0.612 & 0.569 & 3.5e-2
        & 0.376 & 0.809 & 6.8e-2
        & -0.002 & 0.061 & 8.2e-2 \\
        Fast-DataShapley
        & 0.291 & 0.130 & 9.6e-1
        & 0.523 & 0.481 & 4.4e-2
        & 0.477 & 0.246 & 2.7e-2
        & 0.130 & 0.126 & 5.1e-2
        & 0.030 & 0.116 & 7.1e-2 \\
        D-Shap
        & \textbf{0.908} & \textbf{0.819} & \textbf{4.1e-3}
        & \textbf{0.801} & \textbf{0.847} & \textbf{3.4e-4}
        & \textbf{0.870} & \textbf{0.922} & \textbf{3.6e-4}
        & \textbf{0.752} & \textbf{0.905} & \textbf{3.1e-4}
        & \textbf{0.532} & \textbf{0.630} & \textbf{1.2e-3} \\
        \bottomrule
    \end{tabular}}
\end{table}
\begin{table}[t]
    \centering
    \small
    \caption{Player-Incremental Valuation}
    \label{tabplayer}
    \resizebox{\textwidth}{!}{%
    \begin{tabular}{lccccccccccccccc}
        \toprule
        \multirow{2}{*}{Method}
        & \multicolumn{3}{c}{WKNN/MNIST}
        & \multicolumn{3}{c}{DT/Iris}
        & \multicolumn{3}{c}{SVM/BC}
        & \multicolumn{3}{c}{CNN/MNIST}
        & \multicolumn{3}{c}{GNN/Cora} \\
        \cmidrule(lr){2-4}\cmidrule(lr){5-7}\cmidrule(lr){8-10}\cmidrule(lr){11-13}\cmidrule(lr){14-16}
        & $\rho$ & $r$ & $T$
        & $\rho$ & $r$ & $T$
        & $\rho$ & $r$ & $T$
        & $\rho$ & $r$ & $T$
        & $\rho$ & $r$ & $T$\\
        \midrule
        Global-MC$^\dagger$
        & 1.000 & 1.000 & 2.3e4
        & 1.000 & 1.000 & 1.3e3
        & 1.000 & 1.000 & 2.2e4
        & 1.000 & 1.000 & 3.9e5
        & 1.000 & 1.000 & 7.0e5 \\
        TMC-Shapley
        & 0.878 & 0.752 & 8.4e3
        & 0.786 & 0.833 & 9.1e2
        & 0.814 & 0.961 & 6.8e3
        & 0.426 & 0.687 & 1.3e5
        & 0.822 & 0.972 & 7.2e5 \\
        Comple-S
        & 0.425 & 0.710 & 4.5e3
        & 0.719 & 0.912 & 6.7e2
        & 0.498 & 0.879 & 4.2e3
        & 0.439 & 0.675 & 8.6e4
        & 0.629 & 0.837 & 4.6e5 \\
        \midrule
        B-Delta
        & 0.753 & 0.663 & 1.2e4
        & 0.733 & \textbf{0.813} & 1.1e3
        & \textbf{0.698} & 0.657 & 1.3e3
        & 0.318 & 0.280 & 2.1e5
        & 0.643 & \textbf{0.829} & 5.1e5 \\
        D-Shap 
        & \textbf{0.875} & \textbf{0.933} & \textbf{5.9e1}
        & \textbf{0.760} & 0.811 & \textbf{7.5e1}
        & 0.636 & \textbf{0.701} & \textbf{9.3e1}
        & \textbf{0.462} & \textbf{0.680} & \textbf{3.9e2}
        & \textbf{0.705} & 0.697 & \textbf{2.5e2} \\
        \bottomrule
    \end{tabular}}
\end{table}
\begin{table*}[t]
\centering
\small
\newlength{\subtabwA}
\newlength{\subtabwB}
\newlength{\subtabwC}
\setlength{\subtabwA}{0.38\textwidth}
\setlength{\subtabwB}{0.29\textwidth}
\setlength{\subtabwC}{0.28\textwidth}
\begin{minipage}[t]{\subtabwA}
    \centering
    \captionof{table}{Utility Locality}
    \label{tablocality}
    \resizebox{\linewidth}{!}{%
    \begin{tabular}{@{}lcccc@{}}
        \toprule
        Model & Feature $\rho$ & D-Shap $\rho$ & Feature $r$ & D-Shap $r$ \\
        \midrule
        WKNN & 0.685    & \textbf{0.908} & 0.621 & \textbf{0.819} \\
        DT   & 0.758    & \textbf{0.800} & 0.794 & \textbf{0.847} \\
        SVM  & 0.807    & \textbf{0.870} & 0.913 & \textbf{0.922} \\
        CNN  & 0.596    & \textbf{0.752} & 0.876 & \textbf{0.905} \\
        GNN  & $-0.003$ & \textbf{0.532} & 0.006 & \textbf{0.630} \\
        \bottomrule
    \end{tabular}
    }
\end{minipage}
\hspace{0.01\textwidth}
\begin{minipage}[t]{\subtabwB}
    \centering
    \captionof{table}{Coalition Locality}
    \label{tabsupport}
    \resizebox{\linewidth}{!}{%
    \begin{tabular}{@{}lllcc@{}}
        \toprule
        Dataset & Pred. & Support & $\rho$ & $r$ \\
        \midrule
        MNIST & WKNN & WKNN & \textbf{0.875} & \textbf{0.933} \\
        MNIST & WKNN & SVM  & 0.534          & 0.716          \\
        MNIST & WKNN & DT   & 0.658          & 0.390          \\
        \midrule
        Cora  & GNN  & GNN  & \textbf{0.705} & \textbf{0.697} \\
        Cora  & GNN  & WKNN & 0.308          & 0.491          \\
        Cora  & GNN  & DT   & 0.417          & 0.510          \\
        \bottomrule
    \end{tabular}
    }
\end{minipage}
\hspace{0.01\textwidth}
\begin{minipage}[t]{\subtabwC}
    \centering
    \captionof{table}{Matrix Efficiency}
    \label{tab:self_val_construction}
    \resizebox{\linewidth}{!}{%
    \begin{tabular}{@{}lccc@{}}
        \toprule
        Model & Naive & D-Shap & Speedup \\
        \midrule
        WKNN & 4102 s  & 149 s  & 27.6$\times$ \\
        DT   & 12724 s & 251 s  & 50.7$\times$ \\
        SVM  & 9165 s  & 262 s  & 35.0$\times$ \\
        CNN  & $>$100 hrs & 5410 s & $>$66.5$\times$ \\
        GNN  & $>$100 hrs & 3432 s & $>$104.9$\times$ \\
        \bottomrule
    \end{tabular}
    }
\end{minipage}
\end{table*}

\subsection{Task-Incremental Valuation (RQ1)}
\label{sec:exp_task}

Table~\ref{tabtest} evaluates D-Shap on task-incremental valuation. D-Shap achieves accuracy comparable to static baselines while reducing per-task cost by $10^6$--$10^8\times$, enabling real-time valuation. This is achieved by replacing recomputation with interpolation over existing columns under utility locality. Learning-based baselines assume feature similarity implies valuation similarity. However, Proposition~\ref{prop:test} shows that valuation similarity is governed by model-induced computation structure. When these diverge, learning-based methods fail. This is most evident in GNNs, where feature-based baselines collapse, while D-Shap reaches 0.532. Although performance on GNNs is weaker due to limitations of $d_\Gamma$, D-Shap consistently outperforms learning-based methods.

\subsection{Player-Incremental Valuation (RQ2)}
\label{sec:exp_player}

Table~\ref{tabplayer} evaluates D-Shap on player-incremental valuation. D-Shap maintains comparable accuracy while reducing update cost by $10^1$--$10^3\times$. This efficiency comes from exploiting coalition locality: D-Shap identifies the affected set $\mathcal{R}(z')$ and updates only the corresponding local games using Theorem~\ref{thm:player} and Corollary~\ref{cor:monotone}, while reusing all unaffected entries. In contrast, dynamic baselines still require repeated Monte Carlo updates, whose cost grows with dataset size, making them expensive for deep models. In general, D-Shap avoids global recomputation by converting player updates into localized matrix updates, achieving up to $1000\times$ speedup and maintaining competitive quality.

\subsection{Role of Model-Induced Locality (RQ3)}
\label{sec:exp_locality}
We evaluate the necessity of \textbf{model-induced locality} from two perspectives. First, Table~\ref{tablocality} analyzes utility locality by comparing the model-induced distance $d_\Gamma$ with $L_2$ feature-space proximity for task-incremental valuation. Across models, $d_\Gamma$ improves $\rho$ by $5.5\%$--$32.6\%$, with the largest gap on GNNs, where feature proximity fails ($\rho=-0.003$) while $d_\Gamma$ reaches $0.532$; even for WKNNs, $d_\Gamma$ improves $\rho$ from $0.685$ to $0.908$. Second, Table~\ref{tabsupport} studies coalition locality by testing whether the support-set definition $\mathcal{N}(t)$ aligns with the prediction model in player-incremental valuation. Matched support performs best, improving $\rho$ by $33.0\%$--$63.9\%$ on WKNNs and by $69.1\%$--$128.9\%$ on GNNs, while mismatched support captures only partial task-specific utility. The sensitivity to support-set size is provided in Appendix~\ref{apx:exp:locality}.

\subsection{Efficiency of Shapley Matrix Construction (RQ4)}
\label{sec:exp_construction}

Table~\ref{tab:self_val_construction} evaluates the cost of constructing the Shapley matrix $\boldsymbol{\Phi}$. D-Shap achieves over $25\times$ speedup across all settings, with even larger gains for deep models, where model training is more expensive and naive methods fail to finish within $100$ hours. This efficiency comes from shared subset scheduling, which removes redundant evaluations. 
Appendix~\ref{apx:exp:anchor} studies coverage-aware anchor selection, which further reduces matrix construction cost while preserving update quality. These results show that D-Shap makes Shapley matrix construction practical.

\section{Related Work}
\label{sec:related}
Shapley values originate from cooperative game theory~\cite{shapley1953value} and were introduced to ML data valuation by Data Shapley~\cite{ghorbani2019data}, where training instances are treated as players and utility is measured on a held-out validation set. They have since been used for feature attribution~\cite{lundberg2017unified}, data selection~\cite{ghorbani2022data}, model debugging~\cite{adebayo2020debugging}, and data marketplace pricing~\cite{liu2021dealer,agarwal2019marketplace}. To reduce estimation cost, prior work uses Monte Carlo coalition sampling~\cite{Castro2009Polynomial}, truncated coalition evaluation~\cite{ghorbani2019data}, complementary-contribution averaging~\cite{sun2024shapley}, and group
testing~\cite{jia2019towards}, or exploits model structure via
TreeSHAP~\cite{Lundberg2020TreeSHAP}, KNN-Shapley~\cite{jia2019efficient},
and Local Shapley~\cite{yang2026local}. Gradient- and influence-based
proxies~\cite{koh2017understanding,basu2020influence,park2023trak,wang2024data} trade exact Shapley semantics for tractable per-trajectory or per-update
attribution. All of the above assume a fixed player-task set; revaluation under change requires re-running the estimator.

As the player set evolves, dynamic Shapley methods~\cite{zhang2023dynamic,xia2025computing} update values for arriving contributors, but still require global recomputation and are limited to fixed tasks. For \emph{task} dynamics, amortized
explainers~\cite{jethani2021fastshap,sun2026fast,covert2024stochastic} train a model once and produce per-task Shapley values via a single forward pass, but require explainer retraining whenever a new player
arrives. In contrast, D-Shap unifies task- and player-incremental valuation by exploiting model-induced locality to update only affected entries in the player-by-task Shapley matrix. It provides principled update rules for both task-side and player-side changes (Propositions~\ref{prop:test} and~\ref{prop:player}) with error guarantees (Theorems~\ref{thm:test} and~\ref{thm:player_approx}), achieving substantial speedup while maintaining valuation quality comparable to full recomputation.

\section{Conclusion}

We introduced \textbf{D-Shap}, a new framework for dynamic Shapley computation that redefines data valuation as a structured matrix maintenance problem. By representing Shapley values as a player-by-task matrix and leveraging model-induced locality, D-Shap enables efficient updates through localized interpolation and restricted recomputation, avoiding costly global recalculation. We further proposed self-valuation, which eliminates the need for pre-specified evaluation tasks and makes the framework fully self-contained and scalable. These ideas transform Shapley computation from an expensive one-shot process into a reusable and incremental procedure, significantly improving practicality in dynamic machine learning systems.


\bibliographystyle{abbrv}
\bibliography{reference}

\clearpage
\appendix
\section{Notations}\label{apx:notation}
\begin{table}[h]
\centering
\small
\setlength{\tabcolsep}{5pt}
\renewcommand{\arraystretch}{1.15}
\caption{Summary of notation.}
\label{tab:notation}
\begin{tabular}{@{}ll@{}}
\toprule
Notation & Description \\
\midrule
$\mathcal{D}=\{z_i\}_{i=1}^{n}$ & Training players, also called players \\
$\mathcal{T}=\{t_j\}_{j=1}^{m}$ & Evaluation tasks \\
$z_i$ & The $i$-th training player or player \\
$t_j$ & The $j$-th task \\
$n$ & Number of training players \\
$m$ & Number of tasks \\
$S\subseteq\mathcal{D}$ & A coalition of players \\
$\theta(S)$ & Model trained on coalition $S$ \\
$g(\theta(S),t)$ & Evaluation function of model $\theta(S)$ on task $t$ \\
$v_t(S)$ & Task-specific utility, $v_t(S)=g(\theta(S),t)$ \\
$\phi(z,t)$ & Shapley value of player $z$ with respect to task $t$ \\
$\boldsymbol{\Phi}\in\mathbb{R}^{n\times m}$ & Player-by-task Shapley matrix \\
$\Phi_{i,j}$ & Entry of $\boldsymbol{\Phi}$, where $\Phi_{i,j}=\phi(z_i,t_j)$ \\
$\boldsymbol{\Phi}_{:,j}$ & Shapley column for task $t_j$ \\
$\boldsymbol{\Phi}_{i,:}$ & Shapley row for player $z_i$ \\
\midrule
$\Delta\mathcal{T}$ & Set of newly arriving tasks \\
$\Delta\mathcal{D}$ & Set of newly arriving players \\
$t'$ & A newly arriving task \\
$z'$ & A newly arriving player \\
$\boldsymbol{\Phi}^{\mathrm{new}}$ & Newly added rows or columns in the updated Shapley matrix \\
$\boldsymbol{\Phi}^{\mathrm{upd}}$ & Updated entries for existing players or tasks \\
\midrule
$\mathcal{N}(t)$ & Local support set of task $t$ before an update \\
$\mathcal{N}^{+}(t)$ & Local support set of task $t$ after inserting a new player \\
$R(z')$ & Affected task set after inserting $z'$ \\
$d_\Gamma(t,t')$ & Model-induced distance between tasks $t$ and $t'$ \\
$\Gamma(t)$ & Local computation structure of task $t$ \\
$L_\Gamma$ & Lipschitz constant with respect to $d_\Gamma$ \\
$\epsilon$ & Coverage radius or locality error tolerance \\
$\widehat{\phi}(t')$ & Interpolated Shapley column for a new task $t'$ \\
$\mathcal{A}(t')$ & Neighboring anchor tasks used to interpolate $\widehat{\phi}(t')$ \\
$w(t',t)$ & Interpolation weight assigned to task $t$ for estimating $t'$ \\
\midrule
$\mathcal{A}$ & Anchor set for self-valuation matrix construction \\
$k$ & Number of selected anchors \\
$k/n$ & Anchor ratio \\
$r_{\max}$ & Maximum covering radius of anchors under $d_\Gamma$ \\
$\tau$ & Threshold for adaptive anchor expansion \\
$K_{\max}$ & Maximum local support size among affected tasks \\
\midrule
$r$ & Pearson correlation with the Global-MC reference \\
$\rho$ & Spearman rank correlation with the Global-MC reference \\
$T$ & Wall-clock update or construction time \\
\bottomrule
\end{tabular}
\end{table}

\clearpage
\section{Detailed Proof}\label{apx:proof}

\subsection{Proof of Lemma~\ref{lem:locality}}
\label{proof:locality}
\begin{lemma}\label{lem:locality}
Let $\bar v_t(\mathcal{S})=v_t(\mathcal{S}\cap\mathcal{N}(t))$ be the localized game, and let $\eta_t = \sup_{\mathcal{S}\subseteq\mathcal{D}} |v_t(\mathcal{S}) - \bar{v}_t(\mathcal{S})|$ denote the localization error. If $\eta_t\le\epsilon$, then for every $z_j\in\mathcal{D}$, $|\phi(z_j,t)-\bar{\phi}(z_j,t)|\le 2\epsilon$, where $\bar{\phi}(z_j,t)$ is the Shapley value under $\bar v_t$.

\end{lemma}

\begin{proof}
By the definition of $\eta_t$, for every coalition $\mathcal{S}\subseteq\mathcal{D}$,
\begin{equation}
\left|v_t(\mathcal{S})-\bar v_t(\mathcal{S})\right|
=
\left|v_t(\mathcal{S})-v_t(\mathcal{S}\cap\mathcal{N}(t))\right|
\le \eta_t
\le \epsilon .
\end{equation}
Fix any training instance $z_j\in\mathcal{D}$. By the Shapley definition,
\begin{equation}
\begin{aligned}
\left|\phi(z_j,t)-\bar{\phi}(z_j,t)\right|
&=
\left|
\frac{1}{n}
\sum_{\mathcal{S}\subseteq\mathcal{D}\setminus\{z_j\}}
\frac{
\bigl[v_t(\mathcal{S}\cup\{z_j\})-v_t(\mathcal{S})\bigr]
-
\bigl[\bar v_t(\mathcal{S}\cup\{z_j\})-\bar v_t(\mathcal{S})\bigr]
}{
\binom{n-1}{|\mathcal{S}|}
}
\right| \\
&\le
\frac{1}{n}
\sum_{\mathcal{S}\subseteq\mathcal{D}\setminus\{z_j\}}
\frac{
\left|v_t(\mathcal{S}\cup\{z_j\})-\bar v_t(\mathcal{S}\cup\{z_j\})\right|
+
\left|v_t(\mathcal{S})-\bar v_t(\mathcal{S})\right|
}{
\binom{n-1}{|\mathcal{S}|}
} \\
&\le
\frac{1}{n}
\sum_{\mathcal{S}\subseteq\mathcal{D}\setminus\{z_j\}}
\frac{2\epsilon}{\binom{n-1}{|\mathcal{S}|}}
=
2\epsilon .
\end{aligned}
\end{equation}
The last equality follows because, for each coalition size $k$, there are $\binom{n-1}{k}$ subsets of $\mathcal{D}\setminus\{z_j\}$ with size $k$, and therefore
\begin{equation}
\frac{1}{n}
\sum_{\mathcal{S}\subseteq\mathcal{D}\setminus\{z_j\}}
\frac{1}{\binom{n-1}{|\mathcal{S}|}}
=
\frac{1}{n}
\sum_{k=0}^{n-1} 1
=
1.
\end{equation}
Thus, $\left|\phi(z_j,t)-\bar{\phi}(z_j,t)\right|\le 2\epsilon$.
\end{proof}

\subsection{Proof of Proposition~\ref{prop:test}}
\begin{customprop}{\ref{prop:test}}
Under Assumption~\ref{asm:lipschitz}, for any two test points $t$ and $t'$,
\begingroup\small
\[
\|\boldsymbol{\phi}(t)-\boldsymbol{\phi}(t')\|_\infty
\le
2L_\Gamma d_\Gamma(t,t').
\]
\endgroup

\end{customprop}

\begin{proof}
Fix any coalition $S \subseteq \mathcal{D}$. By definition,
\[
v_t(S)=g(\theta(S),t).
\]
Since $\theta(S)\in\Theta$, Assumption~\ref{asm:lipschitz} gives
\[
|v_t(S)-v_{t'}(S)|
=
|g(\theta(S),t)-g(\theta(S),t')|
\le
L_\Gamma d_\Gamma(t,t').
\]

Now fix any training instance $z_i\in\mathcal{D}$. By the Shapley definition,
\[
\phi_i(t)
=
\frac{1}{n}
\sum_{S\subseteq \mathcal{D}\setminus\{z_i\}}
\frac{
v_t(S\cup\{z_i\})-v_t(S)
}{
\binom{n-1}{|S|}
}.
\]
Therefore,
\[
\begin{aligned}
|\phi_i(t)-\phi_i(t')|
&\le
\frac{1}{n}
\sum_{S\subseteq \mathcal{D}\setminus\{z_i\}}
\frac{
|v_t(S\cup\{z_i\})-v_{t'}(S\cup\{z_i\})|
+
|v_t(S)-v_{t'}(S)|
}{
\binom{n-1}{|S|}
} \\
&\le
\frac{1}{n}
\sum_{S\subseteq \mathcal{D}\setminus\{z_i\}}
\frac{
2L_\Gamma d_\Gamma(t,t')
}{
\binom{n-1}{|S|}
}.
\end{aligned}
\]
Since the Shapley weights sum to one, we obtain
\[
|\phi_i(t)-\phi_i(t')|
\le
2L_\Gamma d_\Gamma(t,t').
\]
Taking the maximum over all $z_i\in\mathcal{D}$ gives
\[
\|\boldsymbol{\phi}(t)-\boldsymbol{\phi}(t')\|_\infty
\le
2L_\Gamma d_\Gamma(t,t').
\]
\end{proof}

\subsection{Proof of Theorem~\ref{thm:test}}
\begin{customthm}{\ref{thm:test}}
Under Assumption~\ref{asm:lipschitz}, if every reference point $t\in \mathcal{A}(t')$ satisfies $d_\Gamma(t',t) \leq \varepsilon$, then
\begingroup\small
\[
\big\|
\widehat{\boldsymbol{\phi}}(t')
-
\boldsymbol{\phi}(t')
\big\|_\infty
\leq
2L_\Gamma  \varepsilon.
\]
\endgroup

\end{customthm}
\begin{proof}
By the interpolation estimator,
\begin{equation}
\widehat{\boldsymbol{\phi}}(t')
-
\boldsymbol{\phi}(t')
=
\sum_{t\in \mathcal{A}(t')}
w(t',t)
\left(
\boldsymbol{\phi}(t)-\boldsymbol{\phi}(t')
\right),
\end{equation}
where the weights form a convex combination. Therefore,
\begin{align}
\big\|
\widehat{\boldsymbol{\phi}}(t')
-
\boldsymbol{\phi}(t')
\big\|_\infty
&\leq
\sum_{t\in \mathcal{A}(t')}
w(t',t)
\big\|
\boldsymbol{\phi}(t)-\boldsymbol{\phi}(t')
\big\|_\infty \\
&\leq
\sum_{t\in \mathcal{A}(t')}
w(t',t)
2L_\Gamma  d_\Gamma(t',t) \\
&\leq
\sum_{t\in \mathcal{A}(t')}
w(t',t)
2L_\Gamma  \varepsilon \\
&=
2L_\Gamma  \varepsilon.
\end{align}
The second inequality follows from Proposition~\ref{prop:test}, and the last equality follows from $\sum_{t\in \mathcal{A}(t')}w(t',t)=1$. The theorem follows.
\end{proof}

\subsection{Proof of Proposition~\ref{prop:player}}
\begin{customprop}{\ref{prop:player}}
If the model evaluation at $t$ depends only on its local support, then for any $t \notin \mathcal R(z')$, the valuation at $t$ remains unchanged, i.e., $\phi^{+}(z_j,t)=\phi(z_j,t)$ for all $z_j\in\mathcal D$.

\end{customprop}
\begin{proof}
Let $\mathcal D^+=\mathcal D\cup\{z'\}$. For any $t\notin\mathcal R(z')$, by the definition of $\mathcal R(z')$, we have $\mathcal N^+(t)=\mathcal N(t)$. Since $z'$ is not in the original player set $\mathcal D$, this implies $z'\notin\mathcal N^+(t)$. Because the model evaluation at $t$ depends only on its local support, and the local support is unchanged after inserting $z'$, the new player has no effect on the utility at $t$. Thus, for any coalition $A\subseteq\mathcal D^+$,
\begin{equation}
    v_t^{+}(A)=v_t(A\cap\mathcal D).
\end{equation}
In particular, $z'$ is a null player at $t$, and hence
\begin{equation}
    \phi^{+}(z',t)=0.
\end{equation}

It remains to show that the Shapley values of existing players are unchanged. Fix any $z_j\in\mathcal D$ and let $n=|\mathcal D|$. For any $S\subseteq\mathcal D\setminus\{z_j\}$, define
\begin{equation}
    \Delta_t(S;z_j)=v_t(S\cup\{z_j\})-v_t(S).
\end{equation}
Since $z'$ is null at $t$, the same marginal contribution is obtained whether or not $z'$ is included in the coalition:
\begin{equation}
    v_t^{+}(S\cup\{z_j\})-v_t^{+}(S)
    =
    v_t^{+}(S\cup\{z'\}\cup\{z_j\})-v_t^{+}(S\cup\{z'\})
    =
    \Delta_t(S;z_j).
\end{equation}
Therefore, the Shapley value of $z_j$ in the new game is
\begin{align}
    \phi^{+}(z_j,t)
    &=
    \frac{1}{n+1}
    \sum_{S\subseteq\mathcal D\setminus\{z_j\}}
    \Delta_t(S;z_j)
    \left(
        \frac{1}{\binom{n}{|S|}}
        +
        \frac{1}{\binom{n}{|S|+1}}
    \right) \\
    &=
    \frac{1}{n}
    \sum_{S\subseteq\mathcal D\setminus\{z_j\}}
    \frac{\Delta_t(S;z_j)}{\binom{n-1}{|S|}} \\
    &=
    \phi(z_j,t).
\end{align}
Thus, for any $t\notin\mathcal R(z')$, the valuation of every existing player remains unchanged, and the new player has zero value at $t$. The proposition follows.
\end{proof}

\subsection{Proof of Theorem~\ref{thm:player}}
\begin{customthm}{\ref{thm:player}}
Under coalition locality, for any $t \in \mathcal R(z')$, the updated Shapley values are determined by the new local game restricted to $\mathcal{N}^{+}(t)$:
\begingroup\small
\[
\phi^{+}(z,t)
=
\begin{cases}
\displaystyle
\frac{1}{|\mathcal{N}^{+}(t)|}
\sum_{S \subseteq \mathcal{N}^{+}(t)\setminus\{z\}}
\frac{v_t(S \cup \{z\}) - v_t(S)}
{\binom{|\mathcal{N}^{+}(t)|-1}{|S|}},
& z \in \mathcal{N}^{+}(t), \\[12pt]
0,
& z \in \mathcal{N}(t)\setminus \mathcal{N}^{+}(t), \\[6pt]
\phi(z,t),
& z \notin \mathcal{N}(t)\cup\mathcal{N}^{+}(t).
\end{cases}
\]
\endgroup
\end{customthm}
\begin{proof}
Fix an affected test point $t \in \mathcal{R}(z')$. Under coalition locality, the updated utility at $t$ depends only on the players in the new local support $\mathcal{N}^{+}(t)$. Hence, the updated cooperative game can be restricted to the player set $\mathcal{N}^{+}(t)$.

For any player $z \in \mathcal{N}^{+}(t)$, its updated Shapley value is the Shapley value in this restricted game:
\begin{equation}
\phi^{+}(z,t)
=
\frac{1}{|\mathcal{N}^{+}(t)|}
\sum_{S \subseteq \mathcal{N}^{+}(t)\setminus\{z\}}
\frac{v_t(S \cup \{z\}) - v_t(S)}
{\binom{|\mathcal{N}^{+}(t)|-1}{|S|}}.
\end{equation}

If $z \in \mathcal{N}(t)\setminus \mathcal{N}^{+}(t)$, then $z$ belongs to the old local support but is removed from the new local support. Since the updated game at $t$ is restricted to $\mathcal{N}^{+}(t)$, the utility is independent of $z$. Thus, for every coalition $S \subseteq \mathcal{N}^{+}(t)$, we have $v_t(S \cup \{z\}) = v_t(S)$, and the marginal contribution of $z$ is zero. Hence, $\phi^{+}(z,t)=0$.

Finally, if $z \notin \mathcal{N}(t)\cup\mathcal{N}^{+}(t)$, then $z$ is outside both the old and new local supports of $t$. By coalition locality, inserting $z'$ does not change the contribution of such a player to test point $t$. Therefore, $\phi^{+}(z,t)=\phi(z,t)$.

Combining the three cases, the theorem follows.
\end{proof}

\subsection{Proof of Corollary~\ref{cor:monotone}}
\begin{customcor}{\ref{cor:monotone}}
For any $t \in \mathcal{R}(z')$, if $\mathcal{N}^{+}(t) = \mathcal{N}(t) \cup \{z'\}$, then the local update can reuse the original Shapley matrix. Specifically, $z'$ is valued on $\mathcal{N}^{+}(t)$ using coalitions from $\mathcal{N}(t)$, since $\mathcal{N}^{+}(t)\setminus\{z'\}=\mathcal{N}(t)$. Each existing player $z_j \in \mathcal{N}(t)$ is updated by
\begingroup\small
\[
\phi^{+}(z_j, t)
=
\phi(z_j, t) + \Delta_{j,t}(z').
\]
\endgroup
Here, $\Delta_{j,t}(z')$ captures the change in the marginal contribution of $z_j$ caused by inserting $z'$ on $t$.
\end{customcor}
\begin{proof}
Fix any $t \in \mathcal{R}(z')$ such that $\mathcal{N}^{+}(t)=\mathcal{N}(t)\cup\{z'\}$, and let $n_t=|\mathcal{N}(t)|$. By Theorem~\ref{thm:player}, the updated Shapley values are computed in the local game restricted to $\mathcal{N}^{+}(t)$.

For the new player $z'$, the Shapley formula averages its marginal contribution over coalitions that do not contain $z'$. Since
$\mathcal{N}^{+}(t)\setminus\{z'\}=\mathcal{N}(t)$, these coalitions are exactly subsets of the old support $\mathcal{N}(t)$. Hence,
\begin{equation}
\phi^{+}(z',t)
=
\frac{1}{n_t+1}
\sum_{\mathcal{S}\subseteq\mathcal{N}(t)}
\frac{
    v_t(\mathcal{S}\cup\{z'\})-v_t(\mathcal{S})
}{
    \binom{n_t}{|\mathcal{S}|}
}.
\end{equation}
Thus, $z'$ is valued by its marginal effects over coalitions in $\mathcal{N}(t)$.

Next, consider an existing player $z_j \in \mathcal{N}(t)$. For any $\mathcal{S}\subseteq\mathcal{N}(t)\setminus\{z_j\}$, define
\[
m_j(\mathcal{S})
=
v_t(\mathcal{S}\cup\{z_j\})-v_t(\mathcal{S}).
\]
In the updated local game, coalitions excluding $z_j$ either do not contain $z'$ and have the form $\mathcal{S}$, or contain $z'$ and have the form $\mathcal{S}\cup\{z'\}$. Therefore,
\begin{equation}
\phi^{+}(z_j,t)
=
\frac{1}{n_t+1}
\sum_{\mathcal{S}\subseteq\mathcal{N}(t)\setminus\{z_j\}}
\left[
\frac{m_j(\mathcal{S})}{\binom{n_t}{|\mathcal{S}|}}
+
\frac{m_j(\mathcal{S}\cup\{z'\})}{\binom{n_t}{|\mathcal{S}|+1}}
\right].
\end{equation}
The original Shapley value of $z_j$ in the old local game is
\begin{equation}
\phi(z_j,t)
=
\frac{1}{n_t}
\sum_{\mathcal{S}\subseteq\mathcal{N}(t)\setminus\{z_j\}}
\frac{m_j(\mathcal{S})}{\binom{n_t-1}{|\mathcal{S}|}}.
\end{equation}
Using the identity
\[
\frac{1}{n_t\binom{n_t-1}{|\mathcal{S}|}}
=
\frac{1}{n_t+1}
\left(
\frac{1}{\binom{n_t}{|\mathcal{S}|}}
+
\frac{1}{\binom{n_t}{|\mathcal{S}|+1}}
\right),
\]
we can subtract $\phi(z_j,t)$ from $\phi^{+}(z_j,t)$ and obtain
\begin{equation}
\phi^{+}(z_j,t)-\phi(z_j,t)
=
\frac{1}{n_t+1}
\sum_{\mathcal{S}\subseteq\mathcal{N}(t)\setminus\{z_j\}}
\frac{
    m_j(\mathcal{S}\cup\{z'\})-m_j(\mathcal{S})
}{
    \binom{n_t}{|\mathcal{S}|+1}
}.
\end{equation}
Since
\[
m_j(\mathcal{S}\cup\{z'\})-m_j(\mathcal{S})
=
\big[v_t(\mathcal{S}\cup\{z_j,z'\})-v_t(\mathcal{S}\cup\{z'\})\big]
-
\big[v_t(\mathcal{S}\cup\{z_j\})-v_t(\mathcal{S})\big],
\]
this difference is exactly the change in the marginal contribution of $z_j$ caused by inserting $z'$. Defining
\[
\Delta_{j,t}(z')
=
\phi^{+}(z_j,t)-\phi(z_j,t),
\]
we obtain
\[
\phi^{+}(z_j,t)
=
\phi(z_j,t)+\Delta_{j,t}(z').
\]
Therefore, under monotone support expansion, the local update reuses the original Shapley values and updates existing players through an additive correction. The corollary follows.
\end{proof}

\subsection{Proof of Theorem~\ref{thm:player_approx}}
\label{proof:player_approx}
\begin{theorem}\label{thm:player_approx}

Let $\mathcal{D}^{+}=\mathcal{D}\cup\{z'\}$ and let $\mathcal{N}^{+}(t)$ be the local support of $t$ after inserting $z'$. Define the localized updated game
\[
\bar v_t^{+}(\mathcal{S})
=
v_t^{+}(\mathcal{S}\cap \mathcal{N}^{+}(t)),
\qquad
\mathcal{S}\subseteq \mathcal{D}^{+}.
\]
If
\[
\eta_t^{+}
=
\sup_{\mathcal{S}\subseteq\mathcal{D}^{+}}
\left|
v_t^{+}(\mathcal{S})-\bar v_t^{+}(\mathcal{S})
\right|
\le \epsilon,
\]
then for every $z\in\mathcal{D}^{+}$,
\[
\left|
\phi^{+}(z,t)-\bar{\phi}^{+}(z,t)
\right|
\le 2\epsilon,
\]
where $\bar{\phi}^{+}(z,t)$ is the Shapley value under $\bar v_t^{+}$. Equivalently, the local update over $\mathcal{N}^{+}(t)$ approximates the full updated Shapley value within $2\epsilon$.
\end{theorem}
\begin{proof}
The result follows by applying Lemma~\ref{lem:locality} to the updated game with player set $\mathcal{D}^{+}$, utility $v_t^{+}$, and local support $\mathcal N^+(t)$. The condition $\eta_t^{+}\le\epsilon$ is exactly the corresponding utility approximation condition for this updated game. Hence, for every $z\in\mathcal{D}^{+}$, the Shapley value computed under the localized updated game $\bar v_t^{+}$ differs from the full updated Shapley value by at most $2\epsilon$.
\end{proof}

\subsection{Proof of Corollary~\ref{cor:cost}}
\label{proof:cost}
\begin{corollary}\label{cor:cost}
Assume the monotone support expansion condition in Corollary~\ref{cor:monotone}. Then, for updating all affected columns $t\in\mathcal{R}(z')$, the number of new utility evaluations is at most
\begin{equation}
\sum_{t\in\mathcal{R}(z')}2^{|\mathcal{N}(t)|}
\le
|\mathcal{R}(z')|2^{K_{\max}},
\end{equation}
where
\[
K_{\max}=\max_{t\in\mathcal{R}(z')}|\mathcal{N}(t)|.
\]

\end{corollary}
\begin{proof}
Under monotone support expansion, for every affected test point $t\in\mathcal{R}(z')$, the updated support satisfies
\[
\mathcal N^+(t)=\mathcal{N}(t)\cup\{z'\}.
\]
By Corollary~\ref{cor:monotone}, the update only requires evaluating coalitions that contain the new player $z'$ together with a subset of the original support $\mathcal{N}(t)$. Since $\mathcal{N}(t)$ has $2^{|\mathcal{N}(t)|}$ subsets, the number of such coalitions for a fixed affected test point $t$ is at most $2^{|\mathcal{N}(t)|}$.

Summing over all affected test points gives
\[
\sum_{t\in\mathcal{R}(z')}2^{|\mathcal{N}(t)|}.
\]
By the definition of
\[
K_{\max}=\max_{t\in\mathcal{R}(z')}|\mathcal{N}(t)|,
\]
we have $2^{|\mathcal{N}(t)|}\le 2^{K_{\max}}$ for every $t\in\mathcal{R}(z')$. Therefore,
\[
\sum_{t\in\mathcal{R}(z')}2^{|\mathcal{N}(t)|}
\le
\sum_{t\in\mathcal{R}(z')}2^{K_{\max}}
=
|\mathcal{R}(z')|2^{K_{\max}}.
\]
The corollary follows.
\end{proof}

\subsection{Utility Locality under Regularized ERM}
\label{app:erm_utility_locality}

We provide a representative derivation showing that Assumption~\ref{asm:lipschitz} holds under a standard regularized empirical risk minimization setting. 
Let each test point be denoted by $t=(x_t,y_t)$, and let the model trained on a nonempty coalition $S\subseteq\mathcal{D}$ be obtained by
\begingroup\small
\begin{equation}
\theta(S)
=
\arg\min_{\theta}
\frac{1}{|S|}
\sum_{z_i\in S}
\ell(f_\theta(x_i),y_i)
+
\frac{\mu}{2}\|\theta\|_2^2 ,
\label{eq:erm_objective}
\end{equation}
\endgroup
where $\mu>0$ is the regularization strength. 
For the empty coalition, we use a fixed reference model $\theta(\emptyset)$ satisfying the same norm bound stated below. 
Assume that the predictor is linear in a model-induced representation $\psi(t)$:
\begingroup\small
\begin{equation}
f_\theta(x_t)=\theta^\top \psi(t).
\end{equation}
\endgroup
The utility is defined as the negative loss,
\begingroup\small
\begin{equation}
v_t(S)=g(\theta(S),t)=-\ell(f_{\theta(S)}(x_t),y_t).
\end{equation}
\endgroup

\begin{proposition}[Utility locality under regularized ERM]
\label{prop:erm}
Assume that the model prediction has the form
$a_\theta(t)=\langle \theta,\psi(t)\rangle$, where $\psi(t)$ is the local computation representation of $t$. 
Assume that the loss $\ell(a,y)$ is nonnegative and $L_\ell$-Lipschitz in its prediction argument $a$ for every fixed label $y$, and that $\ell(0,y)\le C_0$ for all labels $y$. 
For $\mu$-regularized ERM, any solution $\theta(S)$ for coalition $S\subseteq \mathcal{D}$ satisfies
\begingroup\small
\begin{equation}
\|\theta(S)\|_2
\le
B
:=
\sqrt{\frac{2C_0}{\mu}} .
\end{equation}
\endgroup
Consequently, for any two test points $t$ and $t'$ with the same evaluation label, Assumption~\ref{asm:lipschitz} holds with
\begingroup\small
\begin{equation}
d_\Gamma(t,t')
=
\|\psi(t)-\psi(t')\|_2,
\qquad
L_\Gamma
=
L_\ell B
=
L_\ell\sqrt{\frac{2C_0}{\mu}} .
\end{equation}
\endgroup
\end{proposition}

\begin{proof}
Consider the $\mu$-regularized ERM objective for any coalition $S\subseteq\mathcal{D}$:
\[
    \theta(S)
    \in
    \arg\min_{\theta\in\Theta}
    \frac{1}{|S|}
    \sum_{(x_i,y_i)\in S}
    \ell(a_\theta(x_i),y_i)
    +
    \frac{\mu}{2}\|\theta\|_2^2 .
\]
Assume $a_\theta(t)=\langle \theta,\psi(t)\rangle$. Since $\theta(S)$ minimizes the objective, its objective value is no larger than that of $\theta=0$. Therefore,
\[
    \frac{1}{|S|}
    \sum_{(x_i,y_i)\in S}
    \ell(a_{\theta(S)}(x_i),y_i)
    +
    \frac{\mu}{2}\|\theta(S)\|_2^2
    \le
    \frac{1}{|S|}
    \sum_{(x_i,y_i)\in S}
    \ell(0,y_i).
\]
By nonnegativity of the loss and the assumption $\ell(0,y)\le C_0$, we have
\[
    \frac{\mu}{2}\|\theta(S)\|_2^2
    \le
    C_0.
\]
Thus,
\[
    \|\theta(S)\|_2
    \le
    B
    :=
    \sqrt{\frac{2C_0}{\mu}} .
\]

Now consider two test points $t$ and $t'$ with the same evaluation label $y_t=y_{t'}$. Let the evaluation function be
\[
    g(\theta,t)=-\ell(a_\theta(t),y_t).
\]
Then, for any $\theta(S)$,
\[
\begin{aligned}
    |g(\theta(S),t)-g(\theta(S),t')|
    &=
    |\ell(a_{\theta(S)}(t),y_t)
    -
    \ell(a_{\theta(S)}(t'),y_t)| \\
    &\le
    L_\ell
    |a_{\theta(S)}(t)-a_{\theta(S)}(t')| \\
    &=
    L_\ell
    |\langle \theta(S),\psi(t)-\psi(t')\rangle| \\
    &\le
    L_\ell
    \|\theta(S)\|_2
    \|\psi(t)-\psi(t')\|_2 \\
    &\le
    L_\ell B
    \|\psi(t)-\psi(t')\|_2 .
\end{aligned}
\]
Therefore, Assumption~\ref{asm:lipschitz} holds with
\[
    d_\Gamma(t,t')
    =
    \|\psi(t)-\psi(t')\|_2,
    \qquad
    L_\Gamma
    =
    L_\ell B
    =
    L_\ell\sqrt{\frac{2C_0}{\mu}} .
\]
The proposition follows.
\end{proof}

Proposition~\ref{prop:erm} applies to test points with the same evaluation label. 
When the utility explicitly depends on the evaluation label and $y_t$ may differ from $y_{t'}$, the same argument can be extended by augmenting the model-induced distance with a label-mismatch term. 
For example, if the utility is bounded as $|v_t(S)|\le B_v$ for all $t$ and $S$, then Assumption~\ref{asm:lipschitz} holds under
\begingroup\small
\begin{equation}
d_\Gamma(t,t')
=
\|\psi(t)-\psi(t')\|_2
+
\mathbf{1}[y_t\neq y_{t'}],
\end{equation}
\endgroup
with
\begingroup\small
\begin{equation}
L_\Gamma
=
\max\left\{
L_\ell B,
2B_v
\right\}.
\end{equation}
\endgroup
When $y_t=y_{t'}$, the representation term gives the Lipschitz bound from Proposition~\ref{prop:erm}. 
When $y_t\neq y_{t'}$, the indicator term ensures $d_\Gamma(t,t')\ge 1$, while boundedness gives
$\left|v_t(S)-v_{t'}(S)\right|\le 2B_v$.

Furthermore, the ERM derivation provides one sufficient condition rather than a complete characterization. 
Similar utility-locality bounds can be derived by defining $d_\Gamma$ according to each model's local computation structure, such as neighbor-weight profiles for weighted KNN, kernel-relevance profiles for kernel models, and path or leaf overlap for decision trees. 
When the utility is Lipschitz or bounded with respect to the chosen structure, Assumption~\ref{asm:lipschitz} follows with the corresponding constant. 
For neural models, this reasoning applies most directly when the encoder is fixed or the learned representation is uniformly stable.
\section{Model-Induced Locality}
\label{apx:model_locality}

Here, we define the support set and model-induced distance $d_\Gamma$ used in Section~\ref{sec:locality} for each model family in our experiments. In all cases, $d_\Gamma$ compares the local computation structures that determine the task-specific games $v_t$ and $v_{t'}$, rather than raw input features. The concrete form of $d_\Gamma$ depends on the model family: neighbor relevance for KNN, decision paths for trees, kernel relevance for SVMs, propagation influence for GNNs, and representation-space alignment for deep neural networks.

\subsection{Weighted K-nearest-neighbor classifiers}
For WKNN, the local game at task $t$ is determined by its $K$ nearest training points and their distance-based voting weights. We define $\mathcal{N}(t)$ as the $K$-nearest-neighbor set of $t$. Let $\omega_t(z)\ge 0$ be the relevance weight assigned to training point $z\in\mathcal D$, with $\omega_t(z)=0$ for $z\notin \mathcal{N}(t)$. We set
\[
  d_\Gamma(t,t')
  =
  1
  -
  \frac{\sum_{z\in \mathcal D}\min\{\omega_t(z),\omega_{t'}(z)\}}
       {\sum_{z\in \mathcal D}\max\{\omega_t(z),\omega_{t'}(z)\}}.
\]
This is the weighted Tanimoto distance between the two neighbor-relevance profiles. In the unweighted case, it reduces to Jaccard distance between $K$-neighbor supports.

\subsection{Decision trees}
For a fitted decision tree, the local computation structure of $t$ is its root-to-leaf decision path. We define $\mathcal{N}(t)$ as the set of training points that fall into the same leaf as $t$, and let $\mathrm{Path}(t)$ denote the set of internal nodes on the decision path of $t$. The model-induced distance is
\[
  d_\Gamma(t,t')
  =
  1
  -
  \frac{|\mathrm{Path}(t)\cap\mathrm{Path}(t')|}
       {|\mathrm{Path}(t)\cup\mathrm{Path}(t')|}
  \in[0,1].
\]
This distance captures how much of the recursive partition structure is shared by the two tasks. If two points reach the same leaf, their paths coincide and $d_\Gamma(t,t')=0$.

\subsection{Kernel support-vector machines}
For an RBF-SVM with bandwidth $\gamma>0$ and kernel $K_\gamma(t,z)=\exp(-\gamma\|t-z\|^2)$, the local game at task $t$ is determined by the kernel relevance of training points to $t$. We define $\mathcal{N}(t)$ as the training points with non-negligible kernel relevance, such as the top-$K$ points under $K_\gamma(t,z)$ or those above a fixed threshold. The model-induced distance compares the two kernel-relevance profiles:
\[
  d_\Gamma(t,t')
  =
  1
  -
  \frac{\sum_{z\in \mathcal D}\min\{K_\gamma(t,z),K_\gamma(t',z)\}}
       {\sum_{z\in \mathcal D}\max\{K_\gamma(t,z),K_\gamma(t',z)\}}.
\]
Thus, two tasks are close when they rely on similar training points under the kernel geometry.

\subsection{Graph neural networks}
For an $L$-layer message-passing GNN, the prediction at a node task $t$ depends on training nodes that can influence $t$ through graph propagation. We define $\mathcal{N}(t)$ as the training nodes with non-negligible propagation influence on $t$, implemented by an $L$-hop neighborhood or by retaining the top-influence training nodes. To measure influence, we use the personalized PageRank vector $\boldsymbol{\pi}_t$, centered at $t$ and defined by
\[
  \boldsymbol{\pi}_t
  =
  \alpha \boldsymbol e_t
  +
  (1-\alpha)\widetilde{\boldsymbol A}\boldsymbol{\pi}_t,
\]
where $\widetilde{\boldsymbol A}$ is the normalized adjacency matrix and $\alpha\in(0,1)$ is the teleport probability. The model-induced distance is
\[
  d_\Gamma(t,t')
  =
  1
  -
  \frac{\sum_{z\in \mathcal D}\min\{\boldsymbol{\pi}_t(z),\boldsymbol{\pi}_{t'}(z)\}}
       {\sum_{z\in \mathcal D}\max\{\boldsymbol{\pi}_t(z),\boldsymbol{\pi}_{t'}(z)\}}.
\]
Compared with hard neighborhood overlap, this distance accounts for heterogeneous propagation strength induced by the graph structure.

\subsection{Deep neural networks}
For a deep neural network classifier $f=h\circ\phi$, where $\phi:\mathcal X\to\mathbb R^d$ is a learned representation encoder and $h$ is the prediction head, the local computation structure of task $t$ is represented by its embedding $\phi(t)$. We define $\mathcal{N}(t)$ as the nearest training points to $t$ in the embedding space and measure task similarity by angular alignment:
\[
  d_\Gamma(t,t')
  =
  \frac{1}{2}
  \left(
  1
  -
  \frac{\langle \phi(t),\phi(t')\rangle}
       {\|\phi(t)\|_2\|\phi(t')\|_2}
  \right)
  \in[0,1].
\]
This formulation applies to CNNs, RNNs, and Transformers by taking $\phi(t)$ as the penultimate-layer representation, final hidden state, or pooled token representation, respectively. The factor $1/2$ rescales cosine distance to the unit interval. The encoder $\phi$ is fixed during incremental updates, so $d_\Gamma$ provides a stable geometry for both task-incremental and player-incremental valuation.

\subsection{Label compatibility}
For supervised classification, the utility $v_t(S)=g(\theta(S),t)$ depends on the label of the task. Therefore, two tasks with similar local computation structures but different labels may still induce different cooperative games. When labels are available, we restrict interpolation to label-compatible columns. Equivalently, we set
\[
  d_\Gamma(t,t')=\infty
  \quad
  \text{if } y_t\neq y_{t'}.
\]
This restriction simply reflects that the label is part of the evaluation condition defining $v_t$.
\section{Additional Experimental Results}
\label{apx:experiment}

\subsection{Detailed Setup}
\label{apx:exp:setup}
\begin{table}[h]
    \centering
    \caption{Dataset statistics used across all experiments.}
    \label{tabdataset_stats}
    \small
    \begin{tabular}{lrrrrl}
        \toprule
        Dataset & $|\mathcal{D_\text{train}}|$ & $|\mathcal{D_\text{test}}|$ & Features & Classes & Domain \\
        \midrule
        Iris          & 105   & 45    & 4     & 3 & Tabular \\
        Breast~Cancer & 398   & 171   & 30    & 2 & Tabular \\
        MNIST       & 1{,}000 & 1{,}000 & 784 & 10 & Image  \\
        Cora          & 1{,}708 & 1{,}000   & 1{,}433 & 7 & Graph \\
        \bottomrule
    \end{tabular}
\end{table}

\vparagraph{Datasets and model families.}
Each setting pairs a model family with a dataset whose structure matches the model's locality mechanism, providing a direct test of whether utility and coalition locality hold across qualitatively distinct prediction structures.
\begin{itemize}[leftmargin=1.2em,itemsep=2pt,topsep=2pt]
  \item \textbf{WKNN/MNIST.} We use a weighted $K$-nearest-neighbor classifier with $K{=}5$ over CNN-extracted MNIST image features. The local support $\mathcal{N}(t)$ is the set of $2K{=}10$ nearest training points to $t$ in feature space, chosen to cover the effective neighborhood of the $K$-NN decision rule with a small buffer beyond the $K$ predicting neighbors.
  \item \textbf{Decision Tree/Iris.} We use a depth-limited decision tree trained with the Gini criterion. The local support $\mathcal{N}(t)$ is the set of training points reaching the same leaf as $t$ under the fitted tree, which coincides with the players that determine the prediction at $t$ under the partition structure.
  \item \textbf{SVM/Breast~Cancer.} We use an RBF-kernel support-vector machine with bandwidth tuned by cross-validation on the precomputation set. The local support is $\mathcal{N}(t)=\{z\in\mathcal D:K_\gamma(t,z)\geq 0.5\}$, retaining training points with non-negligible kernel relevance to the test point.
  \item \textbf{CNN/MNIST.} We use a three-layer convolutional network trained on raw MNIST pixel inputs. The local support $\mathcal{N}(t)$ is the top-$20$ training points by penultimate-layer embedding similarity to $\phi(t)$, where $\phi$ is the encoder trained on the precomputation set and held fixed throughout incremental updates.
  \item \textbf{GNN/Cora.} We use a two-layer graph convolutional network on the Cora citation network. The local support $\mathcal{N}(t)$ is the set of training nodes within the two-hop ego network of $t$, matching the receptive field of the two-layer message-passing model.
\end{itemize}
Dataset statistics are summarized in Table~\ref{tabdataset_stats}.

\vparagraph{Evaluation Protocol.}
We use $\mathcal D$ to construct the self-valuation matrix $\boldsymbol{\Phi}$ and reserve a held-out pool that is not used during precomputation. The role of the held-out pool depends on the incremental setting. In task-incremental valuation, the reserved points are treated as incoming prediction tasks while the training set remains fixed. In player-incremental valuation, they are treated as newly arriving training instances that expand the training set, while the evaluation points are kept fixed. We reserve $30\%$ of the data for this held-out pool on Iris, Breast Cancer, and Cora, and sample $1{,}000$ held-out points for MNIST. In all incremental experiments, the held-out pool is processed as a stream, where points are introduced one at a time and the valuation matrix is updated upon each arrival. 

\vparagraph{Baselines.}
We compare D-Shap with both static and dynamic baselines.
\begin{itemize}[leftmargin=1.2em,itemsep=2pt,topsep=2pt]
  \item \textbf{Static baselines.}
  We include three static Shapley estimators that recompute values from scratch after each update. \emph{Global-MC Recompute} is a high-budget Monte Carlo estimator and serves as the reference for quality evaluation. \emph{TMC-Shapley}~\cite{ghorbani2019data} estimates marginal contributions over random permutations and truncates a permutation once the coalition utility has converged, thereby reducing unnecessary evaluations for late-arriving players in a permutation. \emph{Comple-S}~\cite{sun2024shapley} improves sampling efficiency by pairing each sampled coalition with its complementary coalition, so that one sampled subset provides information about two coupled marginal-contribution terms. These methods are general-purpose but do not reuse the existing Shapley matrix, so they still require substantial recomputation under task or player arrivals.
  
  \item \textbf{Task-incremental baselines.}
  For learned Shapley predictors, \emph{Fast-DataShapley}~\cite{sun2026fast} trains a query-conditioned explainer $\phi_\theta(x,y)$, implemented as a two-layer MLP supervised by the Shapley-kernel CWLS loss. We train it for $10^5$ iterations across all settings, requiring more than $10$ hours on CNN/MNIST and GNN/Cora. \emph{Amortized-S}~\cite{covert2024stochastic} predicts scalar Shapley scores via a regressor $f_\theta(t,z)$, implemented as a two-layer MLP that takes a task--player pair as input and is trained on samples from the self-valuation matrix $\Phi^{\mathrm{self}}$ to convergence under early stopping. Reported $T$ for learned predictors is inference-only, and training cost is excluded.

  \item \textbf{Player-incremental baselines.}
  For dynamic shapley valuation, \emph{B-Delta}~\cite{xia2025computing} incrementally updates marginal contribution estimates by reusing previously sampled permutations and updating the terms affected by inserted players. However, its update cost still grows with the number of existing players and sampled permutations, and it does not exploit task-specific coalition locality. We exclude another player-incremental baseline, DeltaShap~\cite{zhang2023dynamic}, from the main comparison because it fails to produce results within our compute budget.
\end{itemize}

\vparagraph{Metrics.}
We measure update quality by comparing each method's final Shapley matrix, after the entire stream of arrivals, against the high-budget Global-MC reference. Let $\widehat{\boldsymbol{\Phi}}$ denote the estimated Shapley matrix and $\boldsymbol{\Phi}^{\star}$ denote the Global-MC reference.

Following the large-game perspective in cooperative game theory, many players in large cooperative games can be individually insignificant, with non-negligible value concentrated on a small subset of influential players~\citep{shapiro1978values}. This phenomenon is also consistent with model-induced locality in data valuation: each task depends primarily on a small local support set, while players outside the support have zero or near-zero contribution. Therefore, when computing correlation metrics, we filter entries whose reference magnitude is below $10^{-3}$, since these near-zero values are dominated by Monte Carlo noise and do not provide reliable ordering information. This filtering removes numerically unstable entries rather than excluding meaningful contributors. Specifically, we evaluate only entries in
\[
  \Omega
  =
  \{(i,j): |\Phi^{\star}_{i,j}|>10^{-3}\}.
\]

\begin{itemize}[leftmargin=1.2em,itemsep=2pt,topsep=2pt]
  \item \textbf{Spearman rank correlation $\rho$.}
  We compute
  \[
    \rho
    =
    \operatorname{corr}
    \left(
    \operatorname{rank}\bigl(\widehat{\Phi}_{\Omega}\bigr),
    \operatorname{rank}\bigl(\Phi^{\star}_{\Omega}\bigr)
    \right),
  \]
  which measures whether a method preserves the relative ordering of Shapley values.

  \item \textbf{Pearson correlation $r$.}
  We compute
  \[
    r
    =
    \frac{
    \sum_{(i,j)\in\Omega}
    \left(\widehat{\Phi}_{i,j}-\bar{\widehat{\Phi}}\right)
    \left(\Phi^{\star}_{i,j}-\bar{\Phi}^{\star}\right)
    }{
    \sqrt{
    \sum_{(i,j)\in\Omega}
    \left(\widehat{\Phi}_{i,j}-\bar{\widehat{\Phi}}\right)^2
    }
    \sqrt{
    \sum_{(i,j)\in\Omega}
    \left(\Phi^{\star}_{i,j}-\bar{\Phi}^{\star}\right)^2
    }
    },
  \]
  where $\bar{\widehat{\Phi}}$ and $\bar{\Phi}^{\star}$ are the averages over $\Omega$. Pearson $r$ measures calibration, i.e., linear agreement between estimated and reference values.

  \item \textbf{Wall-clock time $T$.}
  Following~\cite{zhang2023dynamic}, we report wall-clock time $T$ in seconds, defined as the time required to process an arriving task or training instance. This metric evaluates update efficiency.
\end{itemize}

\vparagraph{Implementation details.}
We use high-budget Global-MC recomputation as the reference for evaluating the quality of all baselines. Following~\cite{ghorbani2019data}, all Monte Carlo methods use the same stopping criterion and are capped at $5{,}000$ samples. Every $100$ samples, we compare the current estimates with those from the previous batch and stop if
\[
\frac{1}{n}\sum_i
\frac{|\phi_i^{(m)}-\phi_i^{(m-100)}|}
     {|\phi_i^{(m)}|+10^{-12}}
<0.05.
\]
This criterion corresponds to an average relative change below $5\%$. Each experiment is repeated with five independent random seeds, and we report the average performance. All experiments are conducted on a single machine with $64$ Intel Xeon E5-2640 v4 CPUs and $64\,\mathrm{GB}$ RAM. For GNN and CNN workloads, we additionally use four NVIDIA RTX A5000 GPUs.
\subsection{Further Analysis of Model-Induced Locality}
\label{apx:exp:locality}

\begin{table}[t]
    \centering
    \small
    \caption{Effect of support set size $|\mathcal{N}(t)|$.}
    \label{tabsize}
    \begin{tabular}{@{}l*{10}{r}@{}}
        \toprule
        \multicolumn{11}{@{}l}{\textit{Task-Incremental}}\\
        \midrule
        \quad $|\mathcal{N}(t)|$ & 2 & 4 & 6 & 8 & 10 & 12 & 14 & 16 & 18 & 20 \\
        \quad $\rho$ & 0.9053 & 0.9078 & 0.9088 & 0.9079 & 0.9075 & 0.9066 & 0.9059 & 0.9047 & 0.9036 & 0.9029 \\
        \quad $r$ & 0.7957 & 0.8121 & 0.8181 & 0.8186 & 0.8189 & 0.8179 & 0.8167 & 0.8147 & 0.8128 & 0.8112 \\
        \quad $T~(10^{-4}\,\mathrm{s})$ & 39.7 & 40.0 & 40.0 & 40.0 & 40.5 & 40.7 & 40.9 & 41.1 & 41.2 & 41.3 \\
        \midrule
        \multicolumn{11}{@{}l}{\textit{Player-Incremental}}\\
        \midrule
        \quad $|\mathcal{N}(t)|$ & 5 & 10 & 15 & 20 & 25 & 30 & 35 & 40 & 45 & 50 \\
        \quad $\rho$ & 0.8589 & 0.8649 & 0.8696 & 0.8745 & 0.8789 & 0.8836 & 0.8886 & 0.8918 & 0.8944 & 0.8975 \\
        \quad $r$ & 0.8065 & 0.8995 & 0.9228 & 0.9329 & 0.9391 & 0.9441 & 0.9476 & 0.9500 & 0.9529 & 0.9554 \\
        \quad $T$ & 16.1 & 28.1 & 42.2 & 58.7 & 76.4 & 97.9 & 122.0 & 148.8 & 179.9 & 203.2 \\
        \bottomrule
    \end{tabular}
\end{table}
While the previous results evaluate the model-induced distance $d_\Gamma$ and support-set alignment, we further analyze the sensitivity of the local support $\mathcal{N}(t)$ to its size. The support size $|\mathcal{N}(t)|$ determines the size of the local coalition game. We sweep $|\mathcal{N}(t)|$ on WKNNs under both task-incremental and player-incremental settings, as shown in Table~\ref{tabsize}. In the task-incremental setting, quality saturates quickly: $\rho$ peaks at $0.9088$ when $|\mathcal{N}(t)|=6$ and changes only slightly as the support grows, while $T$ remains nearly constant at around $4 \times 10^{-3}\,\mathrm{s}$, approximately $4\,\mathrm{ms}$. This indicates that task-incremental interpolation requires only a small neighborhood once the effective local structure is captured. In contrast, player-incremental quality improves steadily with larger support sets, with $\rho$ increasing from $0.8589$ at $|\mathcal{N}(t)|=5$ to $0.8975$ at $|\mathcal{N}(t)|=50$. However, this improvement comes with a clear cost increase, as $T$ grows from $16.1$ seconds to $203.2$ seconds. These results suggest different operating points for the two regimes: a small support, around $|\mathcal{N}(t)|=6$, is sufficient for task-incremental updates, whereas player-incremental updates benefit from larger supports, with $|\mathcal{N}(t)|\in[25,35]$ providing a practical quality-cost trade-off.

\subsection{Anchor Selection Trade-off for Self-Valuation Matrix}
\label{apx:exp:anchor}
\begin{table}[t]
\centering
\small
\setlength{\tabcolsep}{6pt}
\renewcommand{\arraystretch}{1.15}
\caption{Effect of the anchor ratio $k/n$ on self-valuation build time and valuation quality.}
\label{tabanchor}
\begin{tabular}{@{}lccccc@{}}
\toprule
$k/n$ & 0.20 & 0.40 & 0.60 & 0.80 & 1.00 \\
\midrule
$\rho$ & 0.6297 & 0.8923 & 0.9018 & 0.9067 & 0.9090 \\
$r$    & 0.4735 & 0.7559 & 0.7965 & 0.8126 & 0.8189 \\
$T$    & 29.8   & 59.6   & 89.4   & 119.2  & 149.0 \\
\bottomrule
\end{tabular}
\end{table}

Self-valuation construction admits a quality--efficiency trade-off controlled by the anchor ratio $k/n$: smaller anchor ratios reduce the number of columns to precompute and maintain, but provide weaker coverage for task interpolation under Theorem~\ref{thm:test}. We sweep $k/n$ from $0.20$ to $1.00$ on WKNNs and report build time together with task-incremental valuation quality, measured by Pearson $r$ and Spearman $\rho$ against the full-budget Global-MC reference.

Table~\ref{tabanchor} shows a clear trade-off between construction cost and valuation quality. Build time grows linearly with $k/n$, as expected from the per-anchor cost of shared subset scheduling, with each $0.20$ increment adding about $30$ seconds. Quality, in contrast, saturates quickly: $\rho$ rises from $0.6297$ at $k/n=0.20$ to $0.8923$ at $0.40$, while $r$ increases from $0.4735$ to $0.7559$. Beyond this point, additional anchors yield smaller gains, with $\rho$ improving by less than $0.02$ as $k/n$ increases from $0.40$ to $1.00$. This pattern is consistent with the coverage analysis: once anchors are sufficiently dense under $d_\Gamma$, most incoming tasks can be interpolated from nearby precomputed columns, and additional anchors mainly refine an already stable estimate. In practice, $k/n\in[0.4,0.6]$ offers a favorable operating point, achieving near-saturated valuation quality at $40$--$60\%$ of the full-budget construction cost.

\section{Additional Cases for Shapley Matrix Maintenance}
\label{apx:dynamic}

D-Shap also supports deletion, replacement, and simultaneous task--player updates under the same matrix-maintenance principle. The key idea is to treat each dynamic event as a localized update of the player-by-task Shapley matrix $\Phi$: task-side changes modify columns, while player-side changes modify rows and only the local blocks whose support sets change.

\paragraph{Task deletion and replacement.}
For task-side changes, deleting a task $t$ simply removes its corresponding column $\Phi_{:,t}$ from $\Phi$. Since the player set and all other task utilities remain unchanged, no other entries need to be updated. Replacing a task $t$ with a new task $\tilde t$ can be handled as deletion followed by insertion: D-Shap first removes the old column $\Phi_{:,t}$ and then estimates the new task valuation $\hat{\phi}(\tilde t)$ using the interpolation rule in Eq.~\eqref{eqn:task_update}. Thus, task replacement only changes one column of the maintained matrix.

\paragraph{Player deletion and replacement.}
For player-side changes, deleting a player $z$ removes its row from $\Phi$ and updates only tasks whose local support changes. Let $N(t)$ and $N^{-}(t)$ denote the support set of task $t$ before and after deleting $z$, respectively. The affected task set is
\[
R^{-}(z)=\{t\in T: N^{-}(t)\neq N(t)\}.
\]
For any $t\notin R^{-}(z)$, coalition locality implies that the remaining entries are unchanged. For any $t\in R^{-}(z)$, D-Shap recomputes only the local game induced by $N^{-}(t)$ and reuses all entries outside the affected local block. In particular, players in $N^{-}(t)$ are revalued under the updated local game, players removed from the support receive zero contribution, and players outside $N(t)\cup N^{-}(t)$ remain unchanged.

Replacing a player $z$ with $\tilde z$ is handled as deletion followed by insertion. D-Shap first deletes the row corresponding to $z$, updates the tasks in $R^{-}(z)$, and then inserts $\tilde z$ as a new player using the player-incremental update rule in Section~\ref{sec:method}. Equivalently, the affected task set is the union of the tasks affected by deletion and insertion:
\[
R(z\rightarrow \tilde z)
=
\{t\in T: N^{-}(t)\neq N(t)\}
\cup
\{t\in T: N^{+}(t)\neq N^{-}(t)\},
\]
where $N^{+}(t)$ is the support set after inserting $\tilde z$. D-Shap updates only the local games associated with tasks in $R(z\rightarrow \tilde z)$ and reuses all other entries.

\paragraph{Simultaneous task and player arrivals.}
In practice, new tasks and new players may arrive at the same time. Let $\Delta T$ denote the newly arriving tasks and $\Delta D$ denote the newly arriving players. A simple strategy is to process the two updates sequentially. D-Shap first applies the player update to obtain the expanded player set $D^{+}=D\cup \Delta D$ and the updated matrix over existing tasks. It then estimates the columns for new tasks in $\Delta T$ under the updated player set $D^{+}$. This order is preferable when accuracy is prioritized, because the valuation of a new task should be defined with respect to the current player set, including newly arrived players. If task interpolation were performed before inserting new players, the resulting new columns would be estimated under the old game and would still need to be corrected for the newly added players.

More concretely, D-Shap performs the following two-stage update:
\[
\Phi
\;\xrightarrow{\text{player update on } \Delta D}\;
\Phi^{+}_{D,T}
\;\xrightarrow{\text{task update on } \Delta T}\;
\Phi^{+}_{D^{+},T\cup \Delta T}.
\]
The first stage updates only the affected columns
\[
R(\Delta D)=\{t\in T: N^{+}(t)\neq N(t)\},
\]
where $N^{+}(t)$ is computed after inserting $\Delta D$. The second stage estimates each new task column $\hat{\phi}(t')$ for $t'\in \Delta T$ using neighboring columns from the already updated matrix $\Phi^{+}_{D,T}$. Thus, the interpolation anchors and the resulting task valuations are consistent with the expanded player set.

\paragraph{Joint local update for higher efficiency.}
When $\Delta T$ and $\Delta D$ are both small, the sequential update is already efficient. However, for larger update batches, D-Shap can further improve efficiency by using a joint affected-block update. Instead of first updating all affected old tasks and then estimating all new tasks independently, D-Shap forms the joint affected task set
\[
R_{\mathrm{joint}}(\Delta D,\Delta T)
=
R(\Delta D)\cup \Delta T.
\]
For each $t\in R_{\mathrm{joint}}(\Delta D,\Delta T)$, D-Shap constructs the updated support $N^{+}(t)$ under the expanded player set $D^{+}$. It then recomputes exact local Shapley values only when necessary, namely when the task is poorly covered by existing anchors or when the support change is large. Otherwise, it estimates the column by interpolation from updated neighboring columns. This gives the following update rule:
\[
\hat{\phi}^{+}(t)=
\begin{cases}
\mathrm{LocalShapley}(N^{+}(t),t), & \min_{a\in A} d_{\Gamma}(t,a)>\tau
\text{ or } |N^{+}(t)\triangle N(t)|>\kappa,\\
\sum_{a\in A(t)} w(t,a)\phi^{+}(a), & \text{otherwise},
\end{cases}
\]
where $A(t)$ denotes the neighboring anchor tasks under $d_{\Gamma}$, $\tau$ is the coverage threshold, and $\kappa$ controls when support changes are large enough to justify explicit local recomputation.

This joint strategy improves efficiency by avoiding duplicate work. If a new player changes the support of a task that is also used as an interpolation anchor for new tasks, D-Shap updates that anchor once and reuses it for all downstream task interpolations. It can also improve quality because new task columns are interpolated from columns that have already been made consistent with the updated player set. Thus, simultaneous task and player arrivals are handled as a single localized maintenance operation over the affected block of $\Phi$, rather than as global recomputation.

\paragraph{Summary.}
Task addition, deletion, and replacement correspond to adding, removing, or replacing columns of $\Phi$. Player addition, deletion, and replacement correspond to adding, removing, or replacing rows, together with localized recomputation over tasks whose support sets change. When task and player updates occur simultaneously, D-Shap processes player updates first for consistency, or uses a joint affected-block update for better efficiency. Therefore, common dynamic cases are unified as localized maintenance of the player-by-task Shapley matrix.
\clearpage
\section{Pseudocode}\label{app:code}
\begin{algorithm}[H]
\scriptsize
\caption{Task-Incremental Valuation}
\label{alg:task_interpolation}
\begin{algorithmic}[1]
\Require Players $\mathcal D$; anchors $\mathcal A$; Shapley matrix $\Phi_{\mathcal A}^{\mathcal D}$;
distance $d_\Gamma$; new task $t'$; neighborhood size $K$.
\State Find the $K$ nearest anchors
\[
    \mathcal A_K(t')=\operatorname{KNN}_{a\in\mathcal A}(t';d_\Gamma).
\]
\State Compute convex weights $w(t',a)$ over $\mathcal A_K(t')$ with
$w(t',a)\ge0$ and $\sum_{a\in\mathcal A_K(t')}w(t',a)=1$.
\State Estimate
\[
    \widehat{\phi}^{\mathcal D}(t')
    =
    \sum_{a\in\mathcal A_K(t')}
    w(t',a)\,\Phi_{\mathcal A}^{\mathcal D}[:,a].
\]
\State \Return $\widehat{\phi}^{\mathcal D}(t')$.
\end{algorithmic}
\end{algorithm}
\begin{algorithm}[H]
\scriptsize
\caption{Player-Incremental Valuation}
\label{alg:lbu}
\begin{algorithmic}[1]
\Require Training players $\mathcal D$; anchor set $\mathcal A$;
Shapley matrix $\Phi_{\mathcal A}^{\mathcal D}$;
support extractor $\mathcal{N}(\cdot)$; local utility $v_a(\cdot)$;
new player $z'$.
\State Set $\mathcal D^+\leftarrow\mathcal D\cup\{z'\}$ and initialize
$\Phi_{\mathcal A}^{\mathcal D^+}$ by appending a zero row (for $z'$) to
$\Phi_{\mathcal A}^{\mathcal D}$.
\State Compute $\mathcal{N}(a)\subseteq\mathcal D$ and
$\mathcal{N}^+(a)\subseteq\mathcal D^+$ for each $a\in\mathcal A$, and
identify affected anchors
\[
    \mathcal R(z')=\{a\in\mathcal A:\mathcal{N}^+(a)\neq \mathcal{N}(a)\}.
\]
\State For $a\notin\mathcal R(z')$: keep $\Phi_{\mathcal A}^{\mathcal D^+}[\,\cdot\,,a]$ unchanged;
$\Phi_{\mathcal A}^{\mathcal D^+}[z',a]$ is already $0$ from line~1.
\For{each $a\in\mathcal R(z')$}
    \State Set $\Phi_{\mathcal A}^{\mathcal D^+}[z,a]\leftarrow 0$ for $z\in \mathcal{N}^+(a)$. 
    \State Set $\Phi_{\mathcal A}^{\mathcal D^+}[z,a]\leftarrow 0$ for $z\in \mathcal{N}(a)\setminus\mathcal{N}^+(a)$. 
    \ForAll{subsets $S\subseteq \mathcal{N}^+(a)$}
        \State Evaluate $v_a(S)$.
        \ForAll{$z\in \mathcal{N}^+(a)$}
            \If{$z\in S$}
                \State $\Phi_{\mathcal A}^{\mathcal D^+}[z,a]
                \mathrel{+}=
                \dfrac{v_a(S)}{|\mathcal{N}^+(a)|\cdot\binom{|\mathcal{N}^+(a)|-1}{|S|-1}}$
            \Else
                \State $\Phi_{\mathcal A}^{\mathcal D^+}[z,a]
                \mathrel{-}=
                \dfrac{v_a(S)}{|\mathcal{N}^+(a)|\cdot\binom{|\mathcal{N}^+(a)|-1}{|S|}}$
            \EndIf
        \EndFor
    \EndFor
\EndFor
\State \Return $\Phi_{\mathcal A}^{\mathcal D^+}$.
\end{algorithmic}
\end{algorithm}
\begin{algorithm}[H]
\scriptsize
\caption{Self-Valuation Matrix Construction via Shared Subset Scheduling}
\label{alg:self_val_pivot}
\begin{algorithmic}[1]
\Require Training players $\mathcal{D} = \{z_1, \ldots, z_n\}$;
model-induced distance $d_\Gamma$; anchor budget $k$;
local utility $v_a(\cdot)$.
\State Select anchors $\mathcal{A} = \{a_1, \ldots, a_k\} \subseteq \mathcal{D}$ by farthest-point sampling under $d_\Gamma$, fixing an ordering $\Pi_{\mathcal{A}}$.
\State Initialize $\Phi_{\mathcal{A}}^{\mathcal{D}} \in (\mathbb{R} \cup \{\mathrm{NA}\})^{n \times k}$ with
$\Phi_{\mathcal{A}}^{\mathcal{D}}[a, a] \leftarrow \mathrm{NA}$ for $a \in \mathcal{A}$
and $\Phi_{\mathcal{A}}^{\mathcal{D}}[z, a] \leftarrow 0$ otherwise.
\For{each anchor $a \in \mathcal{A}$ in order $\Pi_{\mathcal{A}}$}
    \ForAll{subsets $S \subseteq \mathcal{D} \setminus \{a\}$}
        \State $\mathcal{R}_S \leftarrow \mathcal{A} \setminus S$ 
        \State $a^* \leftarrow$ first anchor in $\mathcal{R}_S$ under $\Pi_{\mathcal{A}}$
        \If{$a = a^*$}
            \State Train $\theta(S)$ and evaluate $v_{a'}(S)$ for all $a' \in \mathcal{R}_S$
            \ForAll{$a' \in \mathcal{R}_S$}
                \ForAll{$z \in \mathcal{D} \setminus \{a'\}$}
                    \If{$z \in S$}
                        \State $\Phi_{\mathcal{A}}^{\mathcal{D}}[z, a'] \mathrel{+}=
                        \dfrac{v_{a'}(S)}{(n-1) \cdot \binom{n-2}{|S| - 1}}$
                    \Else
                        \State $\Phi_{\mathcal{A}}^{\mathcal{D}}[z, a'] \mathrel{-}=
                        \dfrac{v_{a'}(S)}{(n-1) \cdot \binom{n-2}{|S|}}$
                    \EndIf
                \EndFor
            \EndFor
        \EndIf
    \EndFor
\EndFor
\State \Return $\mathcal{A}, \Phi_{\mathcal{A}}^{\mathcal{D}}$.
\end{algorithmic}
\end{algorithm}
\begin{algorithm}[H]
\scriptsize
\caption{Online Anchor Expansion}
\label{alg:adaptive_expansion}
\begin{algorithmic}[1]
\Require Anchor set $\mathcal A$ with ordering $\Pi_{\mathcal A}$; matrix $\Phi_{\mathcal A}^{\mathcal D}$;
distance $d_\Gamma$; expansion threshold $\tau$; incoming task $t'$.
\State Compute coverage radius $r(t';\mathcal A)=\min_{a\in\mathcal A} d_\Gamma(t',a)$.
\If{$r(t';\mathcal A)\le\tau$}
    \State Estimate $\widehat{\phi}^{\mathcal D}(t')$ by Algorithm~\ref{alg:task_interpolation}.
\Else
    \State Compute $\phi^{\mathcal D}(t')$ by exact or high-budget estimation.
    \State Append $\phi^{\mathcal D}(t')$ as a new column to $\Phi_{\mathcal A}^{\mathcal D}$;
    set $\mathcal A\leftarrow\mathcal A\cup\{t'\}$ and append $t'$ to $\Pi_{\mathcal A}$.
\EndIf
\State \Return $\mathcal A,\ \Phi_{\mathcal A}^{\mathcal D}$, valuation for $t'$.
\end{algorithmic}
\end{algorithm}
\clearpage
\end{document}